\documentclass{article}

\usepackage{microtype}
\usepackage{graphicx}
\usepackage{subfigure}
\usepackage{booktabs} 
\usepackage{soul}

\usepackage{hyperref}

\usepackage[accepted]{icml2023}

\usepackage{natbib}
\usepackage{amsmath}
\usepackage{amssymb}
\usepackage{bbm}
\usepackage{comment} 
\usepackage{amsthm}

\usepackage[capitalize,noabbrev]{cleveref}

\newtheorem{definition}{Definition}
\newtheorem{lemma}{Lemma}
\newtheorem{theorem}{Theorem}
\newtheorem{obs}{Observation}
\newtheorem{cor}{Corollary}

\usepackage{mathrsfs}

\newtheorem{innercustomthm}{Theorem}
\newenvironment{customthm}[1]
  {\renewcommand\theinnercustomthm{#1}\innercustomthm}
  {\endinnercustomthm}

\newtheorem{innercustomcor}{Corollary}
\newenvironment{customcor}[1]
  {\renewcommand\theinnercustomcor{#1}\innercustomcor}
  {\endinnercustomcor}
  
\newtheorem{innercustomlemma}{Lemma}
\newenvironment{customlemma}[1]
  {\renewcommand\theinnercustomlemma{#1}\innercustomlemma}
  {\endinnercustomlemma}

\newcommand{\pn}{c_{\delta }k/n}  
\newcommand{\pnp}{p_{\textnormal{UC}} } 
\newcommand{\pnpp}{\bar{p}_{\textnormal{UC}}} 

\newcommand{\pnm}{p^{\prime}} 

\newcommand{\Hard}{ \hat{\mathscr{H}} }
\newcommand{\Easy}{ \mathscr{E}}
\newcommand{\Samp}{ \mathcal{S}}
\newcommand{\Sampn}{ \mathcal{S}_1}
\newcommand{\Sampm}{ \mathcal{S}_2}

\newcommand{\Deltam}{\Delta_{m} }

\newcommand{\cdel}{c_{\delta}}

\usepackage[textsize=tiny]{todonotes}

\icmltitlerunning{Active Learning for $k$-Nearest Neighbors}

\begin{document}

\twocolumn[
\icmltitle{A Two-Stage Active Learning Algorithm for $k$-Nearest Neighbors}




\begin{icmlauthorlist}
\icmlauthor{Nick Rittler}{yyy}
\icmlauthor{Kamalika Chaudhuri}{yyy}
\end{icmlauthorlist}

\icmlaffiliation{yyy}{Department of Computer Science and Engineering, University of California - San Diego, San Diego, CA, USA}

\icmlcorrespondingauthor{Nick Rittler}{nrittler@ucsd.edu}

\icmlkeywords{Machine Learning, ICML}

\vskip 0.3in
]



\printAffiliationsAndNotice{}  

\begin{abstract}
$k$-nearest neighbor classification is a popular non-parametric method because of desirable properties like automatic adaption to distributional scale changes. Unfortunately, it has thus far proved difficult to design active learning strategies for the training of local voting-based classifiers that naturally retain these desirable properties, and hence active learning strategies for $k$-nearest neighbor classification have been conspicuously missing from the literature. In this work, we introduce a simple and intuitive active learning algorithm for the training of $k$-nearest neighbor classifiers, the first in the literature which retains the concept of the $k$-nearest neighbor vote at prediction time. We provide consistency guarantees for a modified $k$-nearest neighbors classifier trained on samples acquired via our scheme, and show that when the conditional probability function $\mathbb{P}(Y=y|X=x)$ is sufficiently smooth and the Tsybakov noise condition holds, our actively trained classifiers converge to the Bayes optimal classifier at a faster asymptotic rate than passively trained $k$-nearest neighbor classifiers.
\end{abstract}

\section{Introduction}
In active learning, the learner has access to unlabeled training data from the target distribution, and can ask an annotator to selectively label a subset of this data. The learner's goal is to develop an accurate classifier while querying as few labels as possible. A significant body of work has demonstrated the power of active learning to lower label complexities versus passive approaches, where every point from the underlying distribution is labeled \citep{balcan2006, hanneke2007, beygelzimer2008, zhang2014}.  In cases where human annotation is a driving cost in supervised training, such guarantees are particularly relevant. 

The previous work on active learning for nearest neighbors algorithms is surprisingly sparse.  \citep{Dasgupta2012} gives asymptotic consistency guarantees for $k$-NN under a wide class of active strategies, but does not provide any finite sample guarantees. \citep{kontorovich2018} uses a compression approach to give guarantees for a $1$-nearest neighbor classifier trained on a subset of an actively acquired training sample; while their algorithm does not use many label queries, it does not come with consistency guarantees ensuring convergence to the Bayes-risk in the large sample limit, and also relies on a somewhat computationally intensive algorithm to determine which labels are to be queried.  \citep{njike2021} gives an algorithm that, similar to \citep{kontorovich2018}, outputs the prediction of a $1$-NN classifier trained on sample points which have been actively determined to be informative. That said, the schemes of both \citep{kontorovich2018} and  \citep{njike2021} do not retain the natural adaptivity to scale changes in the distribution and local de-noising properties of the $k$-NN vote -- which, for one, allow for local estimation of conditional probabilities -- instead resorting to various complex algorithmic subroutines to ensure good performance. Our main contribution is thus the exposition of the only active learning scheme that retains the concept of the $k$-NN vote at prediction time, resulting in a much more direct approach that admits significant convergent rate speed-ups.

Active learning usually relies on some inductive bias or distributional assumption that can exploited to reduce the need for labeled data. The canonical example is that of learning a linear classifier in the realizable case, wherein the a priori knowledge allows the learner to avoid querying examples when one of the possible labelings would be inconsistent with any possible linear classifier and the labels queried so far. The main challenge in actively learning non-parametric classifiers such as nearest neighbors is that this inductive bias is harder to characterize. To facilitate our investigation, we operate under one of the most common distributional assumptions in the non-parametric active learning literature -- a smoothness condition on the conditional probability function $\mathbb{P}(Y=y | X=x)$ \citep{castro2008, minsker2012}. Specifically, we assume that the conditional probability function has a property that is a necessary condition for Hölder continuity under mild regularity assumptions. 

A natural algorithm in this regime is to look at the regions of space where multiple labels can be found, and then sample those regions more closely to get better nearest neighbor estimates. Based on this principle, we propose an intuitive and computationally efficient approach to determine which labels are to be actively queried.

We show that our active learning algorithm comes with natural consistency guarantees -- in the sense that the expected risk of the learned classifier converges to the Bayes-risk as the number of label queries grows to infinity. Additionally, under the Hölder continuity of the conditional probability function $\eta(x) := \mathbb{P}(Y=1|X=x)$, we provide finite sample guarantees for a modified $k$-NN classifier trained on samples acquired via our scheme. We discuss general conditions under which these guarantees are tighter than the corresponding passive guarantees, a discussion which quite notably for the non-parametric active learning literature does not rely on low-noise assumptions.

\begin{figure*}
    \centering
    \includegraphics[width=0.9\textwidth]{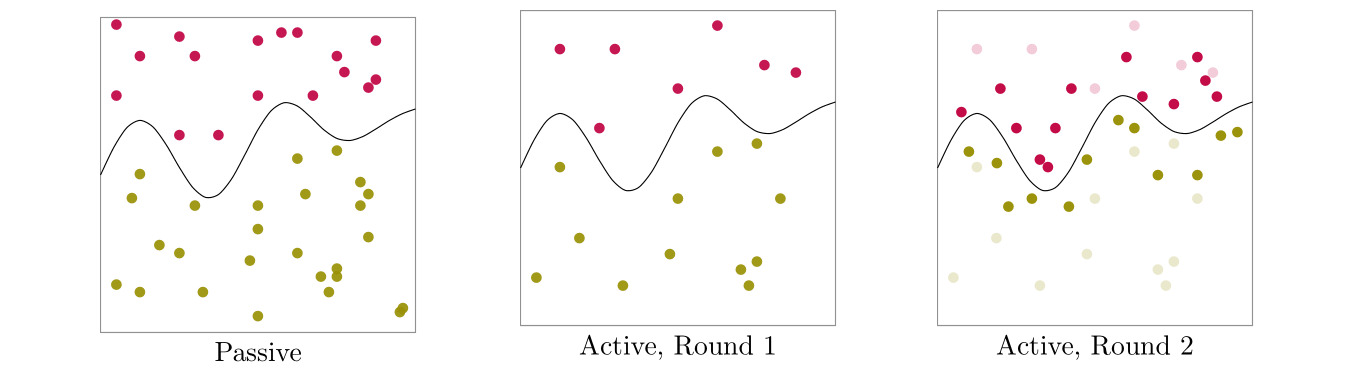}
    \caption{An illustration of the ideal outcome of the two-round active algorithm. In the ``passive'' setting,  points are simply taken from the data-generating distribution. As is often the case in active learning, the goal is to eliminate parts of space where classification is likely easy. In the right-most figure, lower opacity designates first-round samples.}
    \label{fig:img1}
\end{figure*}

We continue by showing that under the Tsybakov noise condition $\textnormal{Pr}_X( |\eta(X) - \frac{1}{2} | \leq t) \leq C t^{\beta}$ \citep{audibert2005}, our actively trained classifiers converge at a significantly faster asymptotic rate than passively trained $k$-NN classifiers. Roughly speaking, assuming $\eta$ is $\alpha$-Hölder continuous on a $d$ dimensional instance space, previous work has shown the fraction of examples on which the learnt classifier disagrees with the Bayes optimal is $\tilde{O}\left(( k^{-1/2} + (k/n)^{\alpha/d})^\beta \right)$ for passively trained $k$-NN after $n$ samples; this improves to $\tilde{O}\left(( k^{-1/2} + ( f_{k,n} \cdot k/n )^{\alpha/d})^\beta \right)$ for $f_{k,n} \in \tilde{O}\left(((d/k)^{1/2} + (k/n + (d/n)^{1/2})^{\alpha/d})^\beta \right)$ for our actively trained classifier.  Our results are a consequence of a new style of guarantee -- controlling the disagreement with the Bayes optimal classifier -- that is novel in the non-parametric active learning literature, and different from the standard work on rates of convergence of expected risk~ \citep{castro2008, locatelli2017}.   We finish with some discussion of the optimality of the scheme.

\section{Related Work}

The study of learning in settings where the learner possesses some power over which samples it uses in training has long been of interest to the learning theory community. 
Though alternative label query models have been explored \citep{angluin1987, awasthi2012}, most attention has been given to the standard active learning model, wherein the learner is assumed to have functionally unlimited access to unlabeled samples from the target distribution, but must choose which of these examples are to be labeled in the course of training a classifier. Much of this work was originally done in the parametric setting, whereby the learner was tasked with choosing an approximately optimal classifier from a hypothesis class \citep{balcan2006, DHM2007, beygelzimer2008}. Because it has long been known that there are settings in which active strategies cannot beat passive ones in terms of label complexity \citep{kaariainen2006}, investigation of specific conditions under which active strategies can beat passive ones has been the norm \citep{balcan2015}. The most general description of settings in which parametric active strategies improve over passive strategies is due to \citep{hanneke2014}.

The study of active learning for the training of non-parametric classifiers has lead to a well-developed theory under smoothness of the conditional probability function and low-noise conditions, which are the common surrogates for inductive bias in the non-parametric setting. Perhaps the most notable work in this litany is \citep{castro2008}. Their work establishes minimax lower bounds on the expected excess risk (over the Bayes-risk) of an actively trained non-parametric classifier, and provides an active scheme for the training of local polynomial classifiers that achieves this lower bound up to log factors. The approximate optimality of their algorithm however assumed access to underlying noise parameters, such as Hölder and Tsbakov parameters. \citep{locatelli2017} resolved this issue with an adaptive algorithm that achieves the same guarantees as \citep{castro2008} without a priori access to noise parameters.  \citep{minsker2012} gives a non-parametric active learning scheme for the training of classifiers based around estimates of the conditional probability function that are locally constant on hypercubes partitioning the instance space. This work is not directly comparable given that $k$-NN estimates of the conditional probability function are locally constant on a more complex geometric structure that is data-dependent.

For the study of active NN algorithms in particular, the authors are not aware of significant contributions other than the aforementioned \citep{kontorovich2018},  \citep{Dasgupta2012} and  \citep{njike2021}.  \cite{BDF2019} gives guarantees for a nearest neighbor algorithm that adaptively sets $k$ across instance space depending on the local agreement of nearest neighbor votes, but this algorithm is passive with respect to data collection.

\section{Preliminaries}
\subsection{Setting}

We study a binary classification setting where instances come from a metric space $(\mathcal{X}, \rho)$, and the label space is $ \mathcal{Y} = \{ 0, 1 \}$. We assume that we have access to labeled data from some data generating measure $\mathbb{P}$ on $\mathcal{X} \times \mathcal{Y}$. $\mathbb{P}$ is the product of some Borel-regular measure $\mu$,\footnote{Thus, $\mu$ can measure all members of the $\sigma$-algebra generated by open sets under $\rho$.} the marginal of $\mathbb{P}$ over the instance space $\mathcal{X}$, and $\eta : \mathcal{X} \to [0,1]$, which denotes the conditional probability function $\mathbb{P}(Y=1| X=x)$. We assume that $\mu$ has full support on $\mathcal{X}$. 

We assume that after $n$ i.i.d. samples from $\mathbb{P}$, which are labeled by definition, we enter a second phase of the sampling with an additional label budget of $m$. In this phase, we may specify any $R \subseteq \mathcal{X}$ with $\mu(R) > 0$, and subsequently rejection sample from $R$, i.e. sample from the product of $\mu$ conditioned on $R$ and $\eta$. The main sampling algorithm, formalized in Algorithm \ref{algo}, makes use of this power to provide a specialized training set for a modified NN classifier, introduced in Section 2.3. Figure 1 gives an idea of the ideal outcome of the strategy.

\subsection{$k$-Nearest Neighbors and Bayes Optimal Classifier}
We use this section to briefly review the standard $k$-nearest neighbors algorithm. Given a set of points in a training set $\mathcal{S} := \{(X_i, Y_i) \in \mathcal{X} \times \mathcal{Y} : i \in [n] \}$, the standard $k$-NN classifier can be stated as 
\begin{align*}
  g_{n, k}(x) := 
  \begin{cases}
    1, & \text{if } \frac{1}{k}\sum_{i=1}^k Y_{(i)}(x) \geq \frac{1}{2} \\
    0, & \text{otherwise } 
  \end{cases}
\end{align*}
\noindent
where $Y_{(i)}(x)$ is the label of this $i^{th}$ nearest training instance from $x$ under $\rho$. In this way, $k$-NN provide a sort of local ``vote'' over what the prediction should be, with prediction at $x$ going to the winning class.

We assume that nearest neighbor tie breaking is done by drawing independent uniform random variables from $[0,1]$, with preference going to lower values of this uniform draw. The event that two sample points are equidistant from a new point $x$ is a measure 0 event (over the draw of $x$) for many distributions, but we will make statements concerning all $x \in \mathcal{X}$ simultaneously, for which tie-breaking will be more impactful; see Appendix A for a primer on how this tie-breaking rule enters the analysis.

The Bayes optimal classifier is simply the best guess for classification given full knowledge of $\eta$, that is 
\begin{align*}
g^{*}(x) := 
  \begin{cases}
    1, & \text{if } \eta(x) \geq \frac{1}{2} \\
    0, & \text{otherwise } 
  \end{cases}
\end{align*}
There is well-established literature on the asymptotic relationship of $k$-NN and $g^{*}$. Perhaps the most intuitively useful is that when $k_n \to \infty$, $\frac{k_n}{n} \to 0$, and the classifier is in Euclidean space,
$$
\mathbb{P}\{ g_{n,k_n}(X) \ne Y \} \to \mathbb{P}\{ g^{*}(X) \ne Y \}
$$
almost surely over the sampling \citep{devroye1994}. Weaker versions exist on general metric spaces \citep{cover1968, stone1977}.

\subsection{A Modified $k$-Nearest Neighbors Classifier}
We introduce and study a modified $k$-NN classifier to facilitate the investigation of our particular setting. The modification helps analytically by making it likely that the nearest neighbors of a new point $x$ form an unbiased estimate of conditional probability in a neighborhood $x$ in the second sampling round, as well as alleviating complications arising from dependencies carried between sampling rounds. For a given sample of size $n+m$,  and a regions $R \subseteq R^+ \subseteq \mathcal{X}$ with $\mu(R^+)>0$, we consider the modified $k$-NN classifier 
\begin{align*}
  g_{n, m, k}(x) := 
  \begin{cases}
    \mathbbm{1}[\frac{1}{k}\sum_{i=1}^k Y_{(i)}(x; R^+) \geq \frac{1}{2}], & \text{if } x \in R \\
    \mathbbm{1}[\frac{1}{k}\sum_{i=1}^k Y_{(i)}(x) \geq \frac{1}{2}], & \text{otherwise } 
  \end{cases}
\end{align*}
Here, we assume that $R^+$ has been used as the acceptance region in the second round of sampling, and its smaller cousin $R$ is the set of points that use rejection samples for prediction. We use $Y_{(i)}(x; R^+)$ to denote the $i^{th}$ nearest neighbor to $x$ out of all  of samples rejection sampled from $R^+$, and use $Y_{(i)}(x)$ to denote the $i^{th}$ nearest neighbor to $x$ out of samples drawn from $\mathbb{P}$. Roughly speaking, the idea will be to treat prediction in $R$ as a prediction problem on a new probability space with properties inherited from the original, larger space. We will see that if the region $R^+$ is well-chosen, and small enough under $\mu$, this targeting procedure has the potential to be effective at improving label complexities. 

\subsection{Finite Sample Guarantees for $k$-NN in the Passive Setting}
The main analytic tools used in this work are based around those introduced in \citep{CD2014}. We use this section to introduce some of the main analytic concepts
of their work, as they are the basis for an intuitive understanding of our two-round sampling algorithm. In the passive setting, their results align with optimal rates of convergence for nonparametric classifiers under standard regularity conditions.

The core idea of their work is that the $k$-NN classifier may disagree with the Bayes optimal classifier in a region of space where the distribution is complex enough that $n$ samples are not sufficient. The measure of this region, which is referred to as the ``effective decision boundary'', is a function of $n$ and $k$, and is a controlling factor in the convergence of $k$-NN to the Bayes optimal classifier. 

The most vital definition of \citep{CD2014}, is the smallest $r$ such that the closed ball $B(x,r) := \{x^{'} \in \mathcal{X}: \rho(x, x^{'}) \leq r \}$ has mass at least $p$.
\begin{definition}
 For $p \in (0, 1]$ the ``probability radius'' of a point $x$ is the real number
$$
r(x; p) := \inf \{r | \mu \left( B(x, r) \right)  \geq p \}.
$$
\end{definition}
\noindent
The power of the probability radius is that it is adaptive to scale changes in $\mu$, and so can uniformly be applied throughout instance space to anchor the definition of a region in which one can be confident that $k$-NN will predict correctly after seeing a certain number of samples. The ``effective decision boundary'' can then be stated as in terms of this quantity.
 \begin{definition}
For any $p \in (0, 1]$ and $\Delta \in (0, \frac{1}{2}]$, the ``effective interior'' is the set
$$
\mathcal{X}^{+/-}_{p, \Delta} = \mathcal{X}^+_{p, \Delta} \cup  \mathcal{X}^-_{p, \Delta}, 
$$
where 
$$
\mathcal{X}^+_{p, \Delta} := \left \{ x \mid \eta(x) > \frac{1}{2} \ \wedge \ \eta(B(x, r)) - \frac{1}{2} \geq \Delta \right \}, 
$$
$$
\mathcal{X}^-_{p, \Delta} := \left \{ x \mid \eta(x) < \frac{1}{2} \ \wedge \  \frac{1}{2} - \eta(B(x, r))  \geq \Delta \right \}.
$$
The  ``effective decision boundary'' is the set 
$$
\partial_{p, \Delta} := \mathcal{X} \setminus \mathcal{X}^{+/-}_{p, \Delta}.
$$
\end{definition}
Here, $\eta(S) := \frac{1}{\mu(S)} \int_S \eta d\mu$, the probability of getting a 1 label conditional on an instance being in $S$ with $\mu(S)>0$. 
\noindent
The main result is that the measure of the ``effective decision boundary'' controls the disagreement rate of the $k$-NN classifier and the Bayes optimal classifier with high probability.
\begin{theorem}[\textbf{of } Chaudhuri and Dasgupta, 2014]

Fix $\delta \in (0, 1)$, and integers $k<n$. With probability $\geq 1-\delta$ over the draw of $n$ i.i.d. training samples from $\mathbb{P}$, 
$$
\textnormal{Pr}_{X \sim \mu} \big( g^{*}(X) \ne g_{n,k}(X) \big) \leq \mu(\partial_{p, \Delta}) + \delta,
$$
when for a near-constant $c_{\delta}$ we have
\begin{align*}
    p &= c_{\delta} \cdot \frac{k}{n+m} \\
    \Delta &= \min\left(\frac{1}{2}, \sqrt{\frac{1}{k}\cdot \log(2/\delta)} \right).
\end{align*}
\end{theorem}
This work will largely focus on the role of the parameter $p$ in determining guarantees, which intuitively corresponds to investigating the ``locality'' of the voting procedure deciding 
the prediction of the classifier. As such, we will fix $\Delta$ to its value in the above theorem statement throughout. The near-constant $c_{\delta}$ will also continue to pop up in our analysis, and so we make it explicit: $c_{\delta} := (1- \sqrt{(4/k) \cdot \log(2/\delta)})^{-1}$. For the regimes we will be interested in, it will be $O(1)$; we reserve further discussion of this for later.

\section{Active Algorithm}

A classic algorithmic idea in both parametric and non-parametric active learning is to iteratively eliminate parts of the instance space where the learner can be sufficiently confident that no more samples are needed. 
Our algorithm operates under this idea as well. However, whereas the learner can directly exploit inductive bias to execute such a strategy in the in parametric setting,  our algorithm implicitly relies on the smoothness of the distribution. We use the first round of sampling, in part, to detect a subset of instance space in which the $k$-NN agree strongly enough that the learner can be confident its predictions in this region are already correct.  The learner then spends most of its labeling budget in the complement of this region in the second round of sampling.

More precisely, we first propose an algorithm that, based on the first round of $n$ samples, detects regions of space where the average of the labels of the $k$-nearest neighbors is close to $1/2$; these are the regions of space where the classifier is not highly confident of the predicted label. We then show that under a natural smoothness assumption on $\eta$, which we describe below, that these are indeed the hard regions -- in other words, the nearest neighbor classifier is expected to be accurate in the complement of these regions. The region from which rejection sample in the second round is a slightly augmented version of these hard regions, expanded so as to contain the nearest neighbors of points in the hard regions as well.

\subsection{Definition of Smoothness}
As discussed earlier, such a strategy inherently requires a smoothness condition for $\eta$, as if the data distribution can exhibit arbitrary behavior, agreement of the $k$-NN votes does not imply anything about the nature of $\eta$ at finer resolutions. The notion of smoothness for $\eta$ under which we operate was introduced in \cite{CD2014}, and is designed to be evocative of traditional notions of Hölder continuity in more general metric spaces.

\begin{definition}
For $\alpha>0, L>0$, we say the conditional probability function $\eta$ is ``$(\alpha, L)$-smooth'' in $(\mathcal{X}, \rho, \mu)$ if for all $x, x^{'} \in \mathcal{X}$, 
$$
\left |\eta(x) - \eta(x^{'}) \right| \leq L \mu\left(B^{o}(x, \rho(x, x^{'}))\right)^{\alpha}, 
$$
where $B^{o}(x, \rho(x, x^{'}))$ denotes the open ball centered at  $x$ with radius $ \rho(x, x^{'})$.
\end{definition}
When $\mathcal{X} = \mathbb{R}^d$ and $\mu$ has a density with respect to the Lebesgue measure that is non-trivially bounded below, this is a necessary condition for the Hölder continuity of $\eta$ \cite{CD2014}. 
 
\subsection{What Regions are Hard?}
When smoothness of $\eta$ holds, it is possible to construct empirical estimates of the regions of space in which $k$-NN needs extra samples to classify accurately. We use the following region as the basis for the area of instance space in which the learner uses the majority of its label budget in the second round of sampling.

\begin{definition}
Given a confidence parameter $\delta \in (0,1)$, and a first-round sample $\mathcal{S}_1$, the ``estimated hard region'' after $n$ samples for the $k$-NN classifier is the set
$$
\hat{\mathscr{H}}_{n,k} :=\left \{ x \in \mathcal{X} \mid \left |\frac{1}{k} \sum_{i=1}^k Y_{(i)}(x)-\frac{1}{2} \right| < 3\Delta_{\textnormal{UC}} \right \}, 
$$
where 
$$
\Delta_{\textnormal{UC}} := c_0 \sqrt{\frac{d \log(n) + \log(1/\delta)}{k} }.
$$
Here, $c_0$ is a universal constant, and $d$ denotes the VC-dimension of the set of all sets of the form 
\begin{align*}
 B'(x', r', z') &:= \bigg \{(x, z) \in \mathcal{X} \times [0,1] \mid \\ 
 & \hspace{7mm} \rho(x, x')< r  \veebar \ \rho(x, x')=r, z< z' \bigg\}.
\end{align*}.
\end{definition}
The set of $ B'(x', r', z')$ is a natural object of analysis for nearest neighbor tie-breaking. Some standard settings where $d$ is on the order of the VC-dimension of the set of closed balls are discussed in the Appendix.

 \subsection{Utility of the Estimated Hard Region}
 The ``estimated hard region'' has some extremely useful properties. As mentioned in Section 2.4, the complement of the effective decision boundary for passive nearest neighbors can be thought of as the region of space where the learner can be confident that $k$-NN will predict correctly after $n$ samples. It is a fundamental lemma of ours that the ``estimated hard region'' approximates the effective boundary from the outside when $\eta$ is sufficiently smooth. 

\begin{lemma}
For all $k<n \in \mathbb{N}$, and all $\delta \in (0,1)$, if $\eta$ is sufficiently $(\alpha, L)$-smooth in the sense that
$$
2L \bigg(  \frac{k}{n} + c_2\sqrt{\frac{ d + \log(1/\delta) }{n}}  \bigg)^\alpha \leq \Delta_{\textnormal{UC}},   
$$
for some absolutely constant $c_2$, then with probability $\geq 1-\delta^2/16$ over the draw of $n$ $i.i.d.$ samples from $\mathbb{P}$, 
$$
\partial_{c_{\delta} \frac{k}{n}, \Delta} \subseteq \hat{\mathscr{H}}_{n, k}.
$$
\end{lemma}
In other words, it's very likely that after the first round of sampling, the complement of the estimated hard region will be a subset of instance space in which the $k$-NN does not need more samples to classify well. 

 \subsection{Specification of the Rejection Region}
The the actual region used for rejection sampling is a larger version of this set which facilitates the analysis. The idea is that by augmenting the hard region, we can be confident that the nearest neighbors of any point in the estimated hard region are contained in this augmented region. This ensures that the expectation of the nearest neighbors vote is the conditional probability of 1 in nearest neighbors balls, an important property in our analysis.

\begin{definition} 
Given a sample $\mathcal{S}_1$ of size $n$ and a parameter $\zeta \geq 0$,  the ``augmented estimated hard region'' is the set 
$$
\hat{\mathscr{H}}^+_{n,k, \zeta} := \bigcup_{x \in \hat{\mathscr{H}}_{n,k}} B^{o}_{(\bar{k}_\zeta + 1)}(x), 
$$
where 
$$
\bar{k}_\zeta =  \bigg \lceil   c_{\delta} (1+ \zeta) k  +  \Theta \bigg ( \sqrt{ n \left( \log(4/ \delta) + d \right)}  \bigg)  \bigg \rceil, 
$$
and $B^{o}_{(\bar{k}_{\zeta} + 1)}(x)$ denotes the open ball of centered at $x$ with radius distance from $x$ to $\bar{k}_{\zeta}+1$-NN of $x$ under $\rho$.
\end{definition}
Thus, the augmented hard region contains all the points in the hard region and points close by to these points. We reserve a more precise discussion of the constants for the Appendix. We note that as a union of open sets, this is a measurable region for Borel regular $\mu$.

The parameter $\zeta$ guards against the discontinuity of the distribution. It is intuitive and provably true that  $\mu\left( B\left( x, r(x;p) \right) \right) \geq p$ always. However, to ensure some of our estimation is accurate,  we need a two-sided condition on the measure of $B\left( x, r(x;p) \right)$, which might quite a bit larger than $p$ if $\mu$ is not well-behaved.  The following provides a reasonable way of quantifying such a condition.

\begin{definition}
We say the measure $\mu$ is ``$\zeta$-regular'' if there is a $\zeta \geq 0$ such that for all $x \in \mathcal{X}$ and $p \in (0,1)$, 
$$
\mu\left( B\left( x, r(x;p) \right) \right) \leq (1+\zeta) p.
$$
\end{definition}
That is, the the measure of an $r(x;p)$ is never much greater than $p$. Note when $\mathcal{X} = \mathbb{R}^d$ and $\mu$ possesses a density with respect to the Lebesgue measure that is non-trivially bounded below, this condition is satisfied with $\zeta = 0$.  We give our guarantees for the algorithm in terms of this parameter.

\subsection{Algorithm}
Algorithm \ref{algo} introduces the active learning scheme, which uses the first round of $i.i.d.$ samples from $\mathbb{P}$ to both estimate the hard region of space, and to do classification in its complement. We denote the product of measure $\mathbb{P}$ with itself $n$ times via $\mathbb{P}^{\bigotimes n}$. $\textnormal{Ber}( \cdot )$ is used to denote a Bernoulli random variable with given parameter. All sample statements in the algorithm implicitly mean sample independently of the previous randomness in all cases.

We point out that although the learner can be confident that classification will succeed in the complement of the augmented hard region, the learner labels some points from this region in the second round to ensure consistency of the method. This is encoded in the algorithm via parameter $\pi$, which controls how many samples from the target distribution we take in the second round. We will see that $\pi>0$ is roughly sufficient for the consistency of the method if $n$ is fixed and $m \to \infty$. Without loss of generality, we implicitly assume that $\pi m$ is an integer in our analysis. 

Commencement of the second round of sampling takes place with the checking of the condition $\mu(\hat{\mathscr{H}}^+_{n,k,\zeta})>0$. Note that $\mu(\hat{\mathscr{H}}^+_{n,k, \zeta})$ can be estimated to arbitrary accuracy using only unlabeled samples, which do not count towards label complexity. It is thus not uncommon to see such conditions in the active learning literature \citep{balcan2006, minsker2012}. More importantly, the measure need only be verified to be greater than 0 to proceed into the second round of sampling, which is strictly easier than estimating.\footnote{The ICML version of this algorithm featured a different handling of the case when $\mu(\Hard^+_{n,k,\zeta}) = 0$. In particular, it requests $m$ samples from $\mathbb{P}$ in the second sampling round. This leads to a dependence between the number of samples taken from $\mathbb{P}$ and the probability that local conditional probabilities are well-estimated that was previously unaccounted for in Theorem 2. We note that Lemma \ref{rare_bad_estimation} tells us that with high probability, $\Hard^+_{n,k,\zeta}$ being a null set means that the effective boundary is itself a null set, so this case is not of the utmost interest.}

\begin{algorithm}[tb]
\caption{Two-Round Active Algorithm}\label{algo}
\begin{algorithmic}
\STATE {\bfseries Input:} $k, n, m \in \mathbb{N}$, $\pi \in (0, 1)$, $\zeta \geq 0$,  $\mathcal{S} = \emptyset$, $ i = 0$
\vspace{2mm}
\STATE$\mathcal{S} \gets \{(X_1, Y_1), \dots, (X_n, Y_n) \} \sim \mathbb{P}^{\bigotimes n }$ 
\vspace{1mm}
\IF{$\mu(\hat{\mathscr{H}}^+_{n,k,\zeta}) > 0$}  
\vspace{2mm}
\WHILE{$i < (1-\pi) m$}
\STATE $(X_i, Y_i) \gets  \emptyset$
\STATE Sample $X \sim \mu$
\IF{$X \in \hat{\mathscr{H}}^+_{n,k, \zeta}$}
  	 \STATE $X_i \gets X$
	 \STATE Sample $Y_i  \sim \textnormal{Ber}(\eta(X))$ 
\ENDIF

\STATE $\mathcal{S} \gets \mathcal{S} \cup \{ (X_i,Y_i)\} $
\STATE $i \gets |\mathcal{S}| - n$
\ENDWHILE
\vspace{2mm}
\ENDIF
\STATE $\mathcal{S} \gets \{(X_{n+ 1}, Y_{n+1}), \dots, (X_{ n + \lfloor \pi m \rfloor }, Y_{n+ \lfloor \pi m \rfloor }) \} \sim \mathbb{P}^{\bigotimes {\lfloor \pi m \rfloor} }$ 
\end{algorithmic}
\end{algorithm}

\section{Guarantees}

\subsection{Bounding the Measure of the Acceptance Region}

Looking at Algorithm \ref{algo}, it is easy to see that the scheme will not yield any benefits if the labels disagree significantly everywhere in instance space; in this case, all unlabeled samples will be vacuously labeled in the second round of sampling. In other words, any kind of guarantee of speed up will have to involve a quantity which measures the fraction of the data distribution which is ``easy'' -- and where $\eta$ is bounded away from $1/2$. We begin with defining this region, and subsequently state our core result in terms of its measure.

\begin{definition}
Fix $k<n$ and $\delta \in (0,1)$. We say the ``easy region'' of instance space $\Easy_{n,k}$ is the set defined via 
\begin{align*}
\Easy_{n, k}^+ &:= \left \{x \mid  \forall r \leq r(x; \pnpp),  \ \eta(B(x, r)) - 1/2 \geq 5 \Delta_{\textnormal{UC}} \right\} \\ 
\Easy_{n, k}^- &:= \left \{x \mid  \forall r \leq r(x; \pnpp),  \ 1/2 - \eta(B(x, r))  \geq 5 \Delta_{\textnormal{UC}} \right\} \\
\Easy_{n, k} &:= \Easy_{n, k}^+ \cup \Easy_{n, k}^-.
\end{align*}
where $\pnpp = \frac{\bar{k}_\zeta}{n} + c_2 \sqrt{\frac{d + \log(1/\delta) }{n}}$, for some  constant $c_2$.
\end{definition}
Given the condition bounding $\eta$ away from $1/2$, we will eventually be able to simplify the measure of this set under the Tsybakov noise condition, and further elucidate conditions where speed up occurs.

\subsection{Finite Sample Guarantee}
The core result is a finite sample guarantee for the modified $k$-NN classifier $g_{n, m, k}$ trained on samples acquired in the method 
of Algorithm \ref{algo}, extended to be well-defined in the case $\mu(\Hard^+_{n,k,\zeta})=0$. The style of the result is meant to match Theorem 1 of \citep{CD2014} above, facilitating comparison between the passive and active guarantees. 

In what follows, $\hat{g}$ is the classifier predicting according to the modified $k$-NN classifier  $g_{n,m,k}$  using $R = \Hard_{n,k}$ and $R^+ =  \Hard^+_{n,k,\zeta}$ when $\mu(\Hard^+_{n,k,\zeta})> 0$; in the case that $\mu(\Hard^+_{n,k,\zeta}) = 0$, assume that $\hat{g}$ does standard $k$-NN classification using all of the sampled data for each new point in $\mathcal{X}$.\footnote{The use of vanilla $k$-NN when the proposed acceptance region has no mass was not made explicit in the ICML version.}
\begin{theorem}\label{main}
Fix $\delta \in (0, 1)$, $\zeta \geq 0$,  $\pi \in (0,1)$, and $k, n, m \in \mathbb{N}$,  such that $\bar{k}_{\zeta} < n$ and $m (1-\pi) c_{\delta} \geq n c_{\delta/\sqrt{2}}$. Assume that $\mu$ is $\zeta$-regular, and $\eta$ is $(\alpha, L)$-smooth such that 
$$
3L (\pnpp)^\alpha \leq \Delta_{\textnormal{UC}}.
$$
Then after a draw of at most $n+m$ samples via Algorithm 1 run with parameters $\zeta$ and $\pi$, with probability $\geq 1-\delta$ over the sampling, 
$$
\textnormal{Pr}_{X \sim \mu} \bigg( g^{*}(X) \ne \hat{g}(X) \bigg) \leq \mu(\partial_{p^{'}, \Delta}) + \delta, 
$$
where
$$
p' = c_{\delta/\sqrt{2}} \cdot \frac{k}{ (1-\pi)m} \cdot  \mu(\mathcal{X} \setminus \mathscr{E}_{n,k}).
$$
\end{theorem}
Recall that the main controlling factor in the passive setting, adjusted for the power of $m$ extra samples from $\mathbb{P}$,  is $c_{\delta} \frac{k}{n+m}$. It follows from the definition of the effective boundary that if $p' \leq p$, it holds that $\partial_{p', \Delta} \subseteq \partial_{p, \Delta}$. Thus, fixing the error tolerance $\delta$, a sufficient condition for tighter guarantees than in the passive setting (in addition to the assumptions of Theorem 1) is that 
$$
 c_{\delta/\sqrt{2}}  \frac{k}{ \frac{1-\pi}{\mu(\mathcal{X} \setminus \mathscr{E}_{n,k})}m} \leq c_{\delta} \frac{k}{n+m}.
$$
This condition is a bit opaque, but in Section 4.3, we will see that it admits clean simplification under margin conditions on $\mathbb{P}$.  For now, we note for  $k \in \Omega(d \log(n) )$, it holds that $c_{\delta} \approx c_{\delta/\sqrt{2}}$, and so if in addition $n \approx m$ and $\pi$ is small (which is advantageous for tightening the finite sample bounds of Theorem \ref{main}), an approximate sufficient condition for speed up is $\mu(\mathscr{E}_{n,k}) > \frac{1}{2}$.

\subsection{Consistency}
A major upside of our method is that consistency is almost immediate in wide variety of settings. There are two intuitive reasons for this.  Firstly, we do not permanently cut off any region of space from sampling as $m$ increases and $n$ stays fixed, which is roughly the condition presented by \citep{Dasgupta2012} for making the standard $k$-NN classifier consistent under any active sampling strategy. The second stems from the form of our finite sample guarantees; by controlling the disagreement with the Bayes optimal classifier directly, we naturally converge to the Bayes-risk. To see this second fact, denote the risk of a classifier $g$ via
$$
R(g):= \mathbb{P}\left( g(X) \ne Y \right), 
$$
and the risk of the Bayes optimal classifier be 
$$
R^* :=  \mathbb{P}\left( g^{*}(X) \ne Y \right).
$$
Then, as noted in \citep{CD2014}, for every classifier $g$ it holds that 
$$
R(g) - R^* \leq \textnormal{Pr}_{X \sim \mu} \big(g(X) \ne g^* \big).
$$
Thus, Theorem \ref{main} provides a bound on the excess risk that holds with high probability that under standard regularity conditions shrinks as the sample size increases. This is the main intuition as for why the following guarantee holds.

\begin{theorem}
Suppose $(\mathcal{X}, \rho, \mu)$ satisfies the Lebesgue Differentiation Theorem. Fix $n \in \mathbb{N}$, $\delta \in (0,1)$, $\pi \in (0,1)$ and $\zeta \geq 0$ for use as parameters in Algorithm 1. Suppose the schedule for $k$ respects the following conditions as $m \to \infty$:
\begin{align*}
k_{n+m} &\to  \infty \\
k_{n+m}/ (n+m) & \to 0, 
\end{align*}
and $\bar{k}_{\zeta} < n$ when $k = k_n$. Then for any first round sample $\Sampn$ such that for all $x \in \mathcal{X}$ we have $x \ne X_{(\bar{k}_\zeta + 1)}(x)$, it holds that 
$$
\lim_{m \to \infty} \textnormal{Pr}_{\Sampm} \bigg( R(\hat{g}) - R^{*} > \epsilon \bigg) = 0,  
$$
that is, we converge to the Bayes-risk in probability over the at most $m$ samples drawn according to the second round of Algorithm 1.\footnote{There were a few issues surrounding Theorem 3 in the ICML version. Firstly, it confusingly referred to the set $H_{n,k}$ as the ``acceptance region'', and the notation ``$\Hard_{n,k}^+$'' was used in defining the real-analytic boundary in the subsequent paragraph. Secondly, it did not cover the edge case where a point $x \in \Hard_{n,k}$ simultaneous satisfies $x \notin \Hard_{n,k,\zeta}^+ $, which occurs if and only if all of its $\bar{k}_\zeta \in \Omega(\sqrt{n})$ nearest neighbors are $x$. This has been remedied via the inclusion of the condition ``$x \ne X_{(\bar{k}_\zeta + 1)}(x)$'' in this version. It turns out that a minorly refined analysis shows this condition to be sufficient for the consistency result, and so mention of the real-analytic boundary has been dropped in this version. Finally, there was further an unrelated error in the proof which has been corrected.}
\end{theorem}
Note that this consistency guarantee requires no assumptions on the smoothness of $\eta$, something that is notable for non-parametric active learning algorithms. 

\subsection{Speedup under Margin}
Theorem \ref{main} has the utility of presenting very general conditions under which our active scheme has tighter guarantees than the passively trained $k$-NN classifier. However, with the generality comes a bit of opaqueness. Much of the complexity of the conditions under which our active scheme has tighter guarantees can be cleared up via investigation of Theorem \ref{main} under the canonical Tsybakov noise conditions due to \citep{audibert2005}, which describe distributions with relatively low mass in regions where $\eta$ is near $\frac{1}{2}$. We use the formulation of the noise condition given by \citep{CD2014}. 

\begin{definition}
We say $\mathbb{P}$ satisfies the ``$\beta$-margin condition'' if there is some $C>0$ for which 
$$
\mu \bigg( \left\{ x \mid  \left | \eta(x) - \frac{1}{2} \right | \leq t  \right \} \bigg) \leq C t^{\beta}.
$$ 
\end{definition}

The following corollary to Theorem \ref{main} shows that under smoothness and the $\beta$-margin condition, our actively trained classifier admits tighter high probability bounds on the disagreement with the Bayes optimal classifier than the passive high probability guarantees of \citep{CD2014}. 

\begin{cor}
Suppose the assumptions of Theorem 2 hold are satisfied by $\alpha$, $L$, $\delta$, $\pi$, $\zeta$, $k$ for schedules of first and second round samples $n$, and $m = \Theta(n)$, i.e. for all pairs $(m(i), n(i))$ for $i$ in some unbounded $D \subset \mathbb{N}$. If in addition $\mathbb{P}$ satisfies the $\beta$-margin condition, with probability $\geq 1-\delta$ over the draw of at most $n+m$ training samples from Algorithm 1,
\begin{align*}
&\textnormal{Pr}_{X \sim \mu} \bigg( g^{*}(X) \ne \hat{g}(X) \bigg)  \leq  \\
& \hspace{10mm} \tilde{O} \left( \left( \sqrt{\frac{1}{k}} + \left(\frac{k}{n} \cdot \mu(\mathcal{X} \setminus \Easy_{n,k}) \right)^\alpha \right)^{\beta} \right) + \delta, 
\end{align*}
where 
\begin{align*}
&\mu(\mathcal{X} \setminus \Easy_{n,k}) \leq \tilde{O} \bigg( \bigg( \sqrt{\frac{d\log(n)}{k}} + \left( \frac{k}{n} + \sqrt{\frac{d}{n}} \right)^{\alpha} \bigg)^{\beta} \bigg).
\end{align*}
The corresponding guarantee for the \citep{CD2014} analysis of passively trained nearest neighbors states that when $\beta$-margin condition holds, with probability $\geq 1-\delta$ over the draw of $n+m$ samples from $\mathbb{P}$,  
\begin{align*}
&\textnormal{Pr}_{X \sim \mu} \bigg( g^{*}(X) \ne g_{n + m, k}(X) \bigg)  \leq \\
& \hspace{15mm}  \tilde{O} \left( \left( \sqrt{\frac{1}{k}} + \left(\frac{k}{n}\right)^{\alpha}  \right)^{\beta} \right) + \delta.
\end{align*}
\end{cor}

The $\tilde{O}$ notation here suppresses $\log(1/\delta)$ terms in the bound on the complement of the ``easy'' region, as well as the near-constant $c_{\delta}$; see the Appendix for the precise bound. This showcases that for fixed $\delta$, if $k \in \omega( d \log(n))$,  we have that  $\mu(\mathcal{X} \setminus \Easy_{n,k}) \downarrow 0$, making the active rate tighter than the passive one, and increasingly so as $n+m$ grows. In the above corollary, $k$ can be a function of $n$ as well, but of course must satisfy the conditions of Theorem \ref{main} for each $m, n$ pair in the schedule.

We note that $c_{\delta}$ is $O(1)$ for virtually all regimes where we could expect our active algorithm to improve over passive. For example, if $\delta$ is fixed, then in order to have the estimated hard region guide any targeting over passive sampling, we need $k_{n} > \log(1/\delta) + \log(n)$. For such settings of $k$,  $c_{\delta}$ is bounded above by a constant for sufficiently large $n$. In this case, the $\tilde{O}$ notation is only needed for the bound on $\mu(\mathcal{X} \setminus \Easy_{n,k})$. 


\subsection{Discussion of Optimality of the Scheme}
 The extent to which our scheme is optimal is not fully clear. As noted in our discussion of consistency, the disagreement of a classifier with the Bayes optimal is bounded below by the excess risk of the classifier, but bounds stemming from this observation are not tight, ruling out trivial comparisons to minimax lower bounds over active learners in the literature.

 \citep{CD2014} produces a passive lower bound for passive $k$-NN which states in part that when $\eta$ is $(\alpha, L)$-smooth, $k$-NN disagrees with the Bayes optimal on $x \in \mathcal{X}$ with
$$
 \big | \eta(x) - 1/2 \big | < \frac{1}{\sqrt{k}} - L \left( \frac{k + \sqrt{k} + 1}{n} \right)^{\alpha}, 
 $$
 which for small enough $k$ reduces to 
 $$
  \big | \eta(x) - 1/2 \big | < O \left(\frac{1}{\sqrt{k}} \right). 
 $$
The purpose of the $k$-NN vote is a local estimate of $\eta$, with the prediction improving as the vote becomes more and more local. Roughly speaking, when $\eta$ is smooth but not locally bounded away from $1/2$ by a margin of $\Omega(1/\sqrt{k})$, the vote over $k$ neighbors will fall on the wrong side of $1/2$ (relative to the optimal prediction) with constant probability  \citep{CD2014}. This strongly suggests that the term $1/\sqrt{k}$ in the bound of Corollary 1 is necessary for a voting-based classifier, whether passively or actively trained.

The second term in the guarantee of Corollary 1 corresponds to the locality of the vote; it shrinks as as the number of labels in the first round $n$ increases, and the second round of labeling can be more and more targeted to a specific region of instance space. The factor of $ \sqrt{d\log(n)/k}$ in upper bound on $\mu(\mathcal{X} \setminus \Easy_{n,k})$ is a relic of the use of uniform convergence over conditional probabilities. It is known that this factor cannot be reduced to $ \sqrt{d/k}$, so this term is in some sense necessary for our algorithmic approach of estimating the effective boundary from the outside \citep{BDF2019}. Whether different strategies can show further gains is an interesting future direction.

\section{Conclusion }

In this work, we have introduced a straightforwards and intuitive algorithm for active learning for $k$-NN, showed that it admits simply consistency guarantees, and exhibited conditions under which our guarantees are tighter than corresponding passive guarantees. 

Our algorithm primarily works by improving the locality of the $k$-NN vote in regions where classification with $k$-NN can benefit from higher resolution sampling. An interesting extension to this idea might be to investigate adaptively modifying $k$ throughout the instance space as the learner gains information on which regions of space admit strong classification guarantees, perhaps mirroring the passive strategy of  \citep{BDF2019}.

\section{Acknowledgements}
This work was supported by NSF under CNS 1804829 and ARO MURI W911NF2110317.

\nocite{SS2014}

\bibliography{paper.bib}
\bibliographystyle{icml2023}


\newpage
\appendix
\onecolumn

\section{Nearest Neighbor Tie-Breaking}
We use the tie-breaking rule of  \citep{CD2014} to handle the case that two neighbors of $x$ are the same distance from $x$ under $\rho$. The idea is simply to augment the instance space to $\mathcal{X}' := \mathcal{X} \times [0,1]$, and assume training instances are drawn from the product measure $\mu' := \mu \times \lambda$, where $\lambda$ is the Lebesgue measure on $[0,1]$. Preference then goes to lower values drawn from the uniform measure in tie-breaking. With probability 1 over the randomness in the tie breaking mechanism, ambiguity among nearest neighbors can be resolved in this way.

Given a training sample $\mathcal{S} = \Sampn \cup \Sampm$, and a new point $x \in \mathcal{X}$, let $X_{(i)}(x)$ denote the $i^{th}$ nearest neighbor to $x$ under $\rho$ out of samples drawn from $\mathbb{P}$ (where ties are broken according to the above), and let $Y_{(i)}(x)$ denote the label corresponding to $X_{(i)}(x)$. Let $X'_{(i)}(x) := (X_{(i)}(x), Z_{(i)}(x))$ be the $i^{th}$ nearest neighbor augmented with it's tie breaking draw, i.e. the $i^{th}$ nearest neighbor in the augmented instance space. Our analysis ideas center around balls in the original instance space, but to incorporate this tie breaking mechanism, we often consider balls in the augmented instances space. For $x' \in \mathcal{X}$, $r'\geq0$, and $z' \in [0,1]$, we define 
$$
B{'}(x', r', z') := \left \{ (x, z) \in \mathcal{X}{'} : \rho(x', x) < r' \veebar  \rho(x', x) = r', \ z < z'  \right \}.
$$
For a ball $B' = B'(x', r', z')$ let the conditional probability of a 1 given that an instance falls in the ball be 
$$
\eta(B') := \frac{1}{\mu'(B')} \int_{B'} \eta \ d \mu', 
$$
whenever $\mu'(B') >0$, where we slightly abusively let $\eta(x,z) := \eta(x)$.

\section{Complexity of Augmented Balls}
Our results rely on the VC-dimension of the set of augmented balls in $\mathcal{X} \times [0,1]$, i.e. the VC-dimension of the set 
$$
\mathcal{B}':= \left \{ B'(x', r', z') : x' \in \mathcal{X}, \ r'\geq0, \ z' \in [0,1] \right \}, 
$$
being finite. As such, we find it important to show that this set can indeed have finite VC-dimension. It is intuitive that the VC-dimension of this set is closely related to the VC-dimension of the set of balls in $\mathcal{X}$. While we did not seriously pursue a general result to this end, we illustrate a simple and important example under a specific setting of $\rho$ where the VC-dimension of the set of augmented balls grows linearly with that of balls in the original space $\mathcal{X}$.

\begin{obs}\label{VC_dim_augmented_l1_balls}
Suppose $\mathcal{X} = \mathbb{R}^d$, and $\rho(x, x') = \|x - x'\|_\infty$. 
Then $VC(\mathcal{B}') \leq 2d$. 
\end{obs}

\begin{proof} 
We give an argument based on ideas in \citep{despres2017} and \citep{SS2014}. Let $P' = \{(x_i, z_i) \in \mathbb{R}^d \times [0,1]\}_{i=1}^{2d+1}$ be a set of points in augmented space. For a given $p$ let $p(x) = \left(p(x)(1), \dots p(x)(d) \right)$ denote the values of $x$ in each of the dimensions, and $p(z)$ denote its augmented value, and note that in $l_\infty$, any $B'$ has the form
\begin{align*}
B'(x',r',z') &= \bigg \{ (x,z) \in \mathbb{R}^d \times [0,1] : x(1) \in  (x'(1)-r, x'(1)+r),   \dots, x(d) \in  (x'(d)-r, x'(d)+r) \\
& \hspace{40mm} \veebar \max_i |x(i) - x'(i)| = r' \land z< z' \bigg \}
\end{align*}
Consider ``directions'' denoted ``left'' and ''right''. Now, for each dimension/direction pair, choose some $p \in P$ that is ``extremal'' in this dimension/direction pair in the sense that it holds that $p(l) \geq \max_{p' \ne p} p'(l)$ or $p(l) \leq \min_{p' \ne p} p'(l)$, and WLOG if $p(l) = \max_{p'\ne p} p'(l)$, then $p(z) \geq p'(z)$ for each $p' \ne p$ with $p'(l) = \max_{p \ne p'} p'(l)$, thus forming a list $L$. There is some $p \in P$ such that $p \notin L$ by the fact that are only $2d$ dimension/direction pairs.

It can then be checked from the definitions that there is no ball $B'$ that has $p' \in B'$ for all $p' \in L$ and $p \notin B'$. This is because any ball that contains maximal and minimal points in all dimensions and directions must be large in enough in each dimension to accommodate the maximal and minimal points. Thus, we cannot shatter $2d+1$ points. 
\end{proof}

We note that this analysis essentially gives an $l_1$ result in $\mathbb{R}^2$ as well. Different geometries will likely require more complex arguments. 

We finish this section with the observation that augmented balls have larger VC-dimension than their counterparts in the original instance space. This allows us to use $d = VC(\mathcal{B}')$ throughout in bounds regarding only the uniform convergence of balls in the original instance space $\mathcal{X}$.
\begin{obs}
Suppose $d$ is the VC-dimension of $\mathcal{B}'$ in $\mathcal{X} \times [0,1]$. Then the VC-dimension of the set of closed balls in $\mathcal{X}$, which we denote by $\mathcal{B}$, has $VC(\mathcal{B}) \leq d$.
\end{obs}
\begin{proof}
Suppose a set of points $P$ of size $K$ in $\mathcal{X}$ can be shattered by a set of closed balls $B_1, \dots, B_{2^K}$. Then any augmented balls with $B_1, \dots, B_{2^K}$ as their projection into $\mathcal{X}$ and $z$ parameters $>0$ shatter the original points $P$ each augmented with $z_i=0$ in augmented space.
\end{proof}

\section{Additional Notation for Proofs}
We use this section to tie together some notation from the body with some further notation that will be necessary for our proofs. 

\subsection{Nearest Neighbors}
Given a training sample $\mathcal{S} = \Sampn \cup \Sampm$, we implicitly assume it has been augmented with tie-breaking draws, such that nearest neighbors are well-defined. As introduced above, given a point $x \in \mathcal{X}$, we use $X_{(i)}(x)$ to specifically denote the $i^{th}$ nearest neighbor to $x$ under the metric $\rho$ of samples drawn from $\mathbb{P}$, and let $Y_{(i)}(x)$ denote the label corresponding to $X_{(i)}(x)$. Similarly, let $X_{(i)}(x; \Hard^+_{n,k, \zeta})$ denote the $i^{th}$ nearest neighbor to $x$ drawn as a rejection sample from $\Hard^+_{n,k, \zeta}$, and $Y_{(i)}(x; \Hard^+_{n,k, \zeta})$ it's corresponding label.

\subsection{Balls: Open, Closed, etc.}
We denote closed balls in the original instance space $\mathcal{X}$ via $B(x', r') := \{x \in \mathcal{X} : \rho(x', x) \leq r' \}$ and open balls via $B^o(x', r') := \{x \in \mathcal{X}  : \rho(x', x) < r' \}$. Let 
$$
B_{(k+1)}(x) := B\left(x, \rho\left(x, X_{(k+1)}(x) \right) \right)
$$
 be shorthand for the closed ball of $k+1$ nearest neighbors to $x$ drawn from $\mathbb{P}$ (where again, tie-breaking has already taken place in the augmented dimension). We use a similar notation for nearest neighbor balls of rejection samples: 
 $$
 B_{(k+1)}(x; \Hard^+_{n,k, \zeta}) := B\left(x, \rho\left(x, X_{(k+1)}(x; \Hard^+_{n,k, \zeta}) \right) \right).
 $$ 
 Both definitions may be modified to denote open balls via the inclusion of the ``o'' in the superscript, e.g. 
 $$
 B^o_{(k+1)}(x) := B^o\left(x, \rho\left(x, X_{(k+1)}(x) \right) \right)
 $$
The associated augmented ball is a sort of ``semi-open'' ball that we denote via 
$$
B'_{(k+1)}(x) := B'\left(x, \rho \left(x, X_{(k+1)}(x)\right), Z_{k+1}(x) \right), 
$$
For a given $p \in (0,1]$ and some $x \in \mathcal{X}$, denote the closed ball around $x$ with radius $r(x; p)$ via
$$
B_{p}(x) := B\left (x, r(x;p) \right).
$$

\subsection{Constants and Parameter Settings}
We use $d$ to denote the VC-dimension of balls in the augmented space $\mathcal{X} \times [0,1]$. For given $k, n \in \mathbb{N}$,  some particular parameter settings will also have special significance, and are fixed throughout:
\begin{align*}
\cdel :&= \frac{1}{1 - \sqrt{(4/k) \ln(2/ \delta)}} \\
\Delta &:= \min \bigg( \frac{1}{2}, \sqrt{\frac{\ln(2/\delta)}{k}} \bigg)\\
\Delta_{\textnormal{UC}} &:= c_0\sqrt{\frac{d\log(n) + \log(1/\delta)}{k}} .
\end{align*}
where $c_0$ is a universal constant discussed in Lemma \ref{conditional_UC}. Given an additional setting of $\zeta \geq 0$, the $k$-enlargement $\bar{k}_{\zeta} \in \mathbb{N}$ is given by, for some absolute constant $c_1$ (discussed in Lemma \ref{rp_balls_estimated_from_outside}), 
$$
\bar{k}_\zeta := \left \lceil c_{\delta} (1 + \zeta) k + c_1 \sqrt{ n \left( \log(4/ \delta) + d \right)} \right \rceil - 1.
$$
This is of course equivalent to the floor of sum, but facilitates use in bounds. Then we can also define:
\begin{align*}
\pnp &:= \frac{k}{n} + c_2\sqrt{\frac{ d + \log(1/\delta) }{n}} \\
\pnpp &:= \frac{\bar{k}_\zeta}{n} + c_2 \sqrt{\frac{d + \log(1/\delta) }{n}}, 
\end{align*}
where $c_2$ a universal constant discussed in Lemma \ref{knn_estimates_larger_rp_ball_from_inside}.

\subsection{Empirical Measures}
For a given measurable set $S \subseteq \mathcal{X}$, we use $\hat{\mu}(S)$ to denote the standard empirical of measure of mass of the set under $\mu$, i.e. $\hat{\mu}(S):= \frac{1}{n} \sum_{i =1}^n \mathbbm{1} [X_i \in S]$. Similarly, for a given set $S' \subseteq \mathcal{X}'$ which contains at least one sample instance, we let $\hat{\eta}(S')$ be the standard empirical estimate of the conditional probability of a $1$ in the set $S'$, i.e. 
$$
\hat{\eta}(S'):= \frac{1}{| \{X'_i \in S'\} |} \sum_{i=1}^n Y_i \mathbbm{1} [X'_i \in S'],
$$
where the samples $X'_i$ are in the augmented space. In general, estimates of conditional probabilities take place in the augmented space, where there is no ambiguity about nearest neighbors. In some cases, we will need to distinguish between empirical measures arising from samples from $\mathbb{P}$ and those arising from rejection samples, in which case we will make clear which sample is being used.

\section{Theorem Proofs}
\subsection{Estimation of Probability Balls, Radii}
We begin with some basic results regarding the approximation of $r(x; p)$ uniformly across instance space. These results are the foundation for some results on set containment discussed later.

The first result to this end states that if $\bar{k}_{\zeta}$ is sufficiently large, with high probability, $\bar{k}_{\zeta}$-NN balls approximate $r(x; \pn)$ balls from the outside uniformly across the instance space.

\begin{customlemma}{2}\label{rp_balls_estimated_from_outside}
Fix $k, n \in \mathbb{N}$, and $\delta \in (0,1)$.    Assume $d < \infty$, $\mu$ is $\zeta$-regular, and $\bar{k}_{\zeta}<n$. Then with probability $\geq 1 - \delta^2/32$ over the draw of $n$ i.i.d. samples from $\mathbb{P}$, it holds that for all $x \in \mathcal{X}$ that
$$
B_{\pn}(x) \subseteq B^o_{(\bar{k}_\zeta + 1)}(x).
$$
\end{customlemma}

\begin{proof}
For a given set $S \subseteq \mathcal{X}$, let $H(S) := \sum_{i=1}^n \mathbbm{1} \{ X_i \in S \}$ be the random variable counting the number of first round training samples with instances falling inside $S$. Let $\mathcal{B}$ denote the set of all closed balls in the space $\mathcal{X}$, and let $\mathcal{B}({\pn}) := \{ B_{\pn}(x) \mid x \in \mathcal{X}\}$. Note the relationship between events
$$
\bigg \{ \exists x \in \mathcal{X} \ s.t. \ B_{\pn}(x) \not \subseteq B^o_{(\bar{k}_\zeta +1)}(x) \bigg\} \subseteq  \left \{ \inf_{B \in \mathcal{B}(\pn) } H(B^c) < n - \bar{k}_\zeta \right \}.
 $$
We upper bound the measure of this latter event. Recall that the VC dimension of a collection of sets is the same as that of the collection of complements of the sets in the original collection.
It holds by the definition of $\bar{k}_\zeta$ and uniform convergence over the empirical probabilities of elements of $\mathcal{B}$ (as in \citep{SS2014}) that there is some absolute constant $\tilde{c}$ for which 
\begin{align*}
    \textnormal{Pr}_{\Sampn \sim \mathbb{P}^n} \bigg( \inf_{B \in \mathcal{B}(\pn) } H(B^c) < n - \bar{k}_\zeta \bigg) &\leq \textnormal{Pr}_{\Sampn \sim \mathbb{P}^n} \bigg(\sup_{B \in \mathcal{B}(\pn)} -n\hat{\mu}(B^c)  \geq  \bar{k}_\zeta - n + 1 \bigg) \\
    &\leq \textnormal{Pr}_{\Sampn \sim \mathbb{P}^n} \bigg( \sup_{B \in \mathcal{B}(\pn)} n\mu(B^c) -n\hat{\mu}(B^c) \\
    &\hspace{30mm}  \geq  \left(\bar{k}_\zeta - n +1\right) + n\left(1-(1+\zeta) \pn \right) \bigg) \\
    & \leq \textnormal{Pr}_{\Sampn \sim \mathbb{P}^n} \bigg( \sup_{B \in \mathcal{B}} \left | \mu(B^c) -\hat{\mu}(B^c) \right| \geq  \frac{\bar{k}_\zeta + 1}{n} - (1+\zeta) \pn) \bigg) \\
    & \leq \exp \left(\frac{d}{\tilde{c}} \right) \cdot \exp\left ( - \frac{n}{\tilde{c}} \cdot \left((\bar{k}_\zeta+1)/n -  (1+\zeta)\pn \right)^2  \right) \\
    & \leq \exp \left(\frac{d}{\tilde{c}} \right) \cdot \exp \left(- \frac{c_1^2}{\tilde{c}} \cdot (d + \log(4/\delta) )\right) \\
    & \leq \delta^2/32, 
\end{align*}
where the final line follows from the choice of $c_1$. Again, we have made use of the uniform convergence of empirical measures of balls in space. 
\end{proof}

Next, we show a sort of inverse of the above: we show that for a given $k$, if we pick the probability radius to be large enough, then with high probability, it holds that all $k$-NN balls are contained in balls with this probability radius uniformly across instance space $\mathcal{X}$. The analysis is independent of the choice of $k$, so an analogous statement holds for $\bar{k}_{\zeta}$.
\vspace{5mm}

\begin{customlemma}{3}\label{knn_estimates_larger_rp_ball_from_inside}

Fix $k, n \in \mathbb{N}$ such that $k<n$, and $\delta \in (0,1)$.  Assume $d < \infty$.\ Then with probability $\geq 1 - \delta^2/32$ over the draw of $n$ i.i.d. samples from $\mathbb{P}$, it holds that for all $x \in \mathcal{X}$.
$$
B_{(k+1)}(x) \subseteq B_{\pnp}(x).
$$
\end{customlemma}

\begin{proof}
The proof follows that of Lemma \ref{rp_balls_estimated_from_outside}. As before, let $H(B)$  be the random variable counting the number of training samples with instances falling inside the ball $B$. In this case, we want to make sure this number is not too small.  We again translate events via
$$
\bigg \{ \exists x \in \mathcal{X} \ s.t. \ B_{(k+1)}(x)  \not \subseteq B_{\pnp}(x) \bigg\} \subseteq \left \{ \inf_{B \in \mathcal{B}(\pnp) } H(B) < k+1 \right\},
$$
and upper bound the measure of the latter event via
\begin{align*}
    \textnormal{Pr}_{\Sampn \sim \mathbb{P}^n} \bigg(\inf_{B \in \mathcal{B}(\pnp) } H(B) < k+1 \bigg) &= \textnormal{Pr}_{\Sampn \sim \mathbb{P}^n} \bigg(\sup_{B \in \mathcal{B}(\pnp) } -H(B) \geq- k \bigg) \\ 
    & \leq \textnormal{Pr}_{\Sampn \sim \mathbb{P}^n} \bigg(\sup_{B \in  \mathcal{B}(\pnp) } |\mu(B)- \hat{\mu}(B) | \geq \pnp - k/n \bigg) \\
    & \leq \exp\left( \frac{d}{\tilde{c}} \right) \cdot \exp\left ( - \frac{n}{\tilde{c}} \cdot \left(k/n - \pnp \right)^2   \right) \\ 
    & \leq \delta^2/32, 
\end{align*}
where as in Lemma \ref{rp_balls_estimated_from_outside}, the final inequality comes from a choice of $c_2$. 
\end{proof}

\begin{customlemma}{4}\label{knn_estimates_larger_rp_ball_from_inside2}
Fix $k, n \in \mathbb{N}$  such that $\bar{k}_\zeta<n$, and $\delta \in (0,1)$.  Assume $d < \infty$. Then with probability $\geq 1 - \delta^2/32$ over the draw of $n$ i.i.d. samples from $\mathbb{P}$, 
$$
B_{(\bar{k}_\zeta+1)}(x) \subseteq B_{\pnpp}(x),
$$
for all $x \in \mathcal{X}$.
\end{customlemma}

\begin{proof}
The proof is identical to that of Lemma \ref{knn_estimates_larger_rp_ball_from_inside}, just using $\bar{k}_\zeta$ and $\pnpp$.
\end{proof}

\subsection{Estimation of the Effective Boundary}
In this section, we show how to leverage the smoothness of $\eta$ to show that the estimator $\Hard_{n,k}$ is a useful approximation of the effective boundary $\partial_{c_{\delta}k/n, \Delta}$,  the part of space in which we cannot expect $k$-NN to classify correctly after $n$ samples from $\mathbb{P}$. 

The first result to this end states that the conditional probability of a 1 inside an augmented ball is similar to the conditional probability of a 1 inside a ball in the original space $\mathcal{X}$  when radius of both of the balls is sufficiently small and smoothness on $\eta$ is in effect.

\begin{customlemma}{5}\label{estimation_condition_measure_robust_under_smoothness}

Fix $x \in \mathcal{X}$, $z \in (0,1]$,  $p \in (0, 1)$ and $0 < r, \tilde{r} \leq r(x; p)$ (or $r=0$ or $\tilde{r}=0$ if $\mu(B(x, 0)) > 0$). Then if $(\mathcal{X}, \rho, \mu)$ is $(\alpha, L)$-smooth, it holds that 
$$
\bigg |\eta \left( B'(x, r, z) \right) -\eta \left( B(x, \tilde{r}) \right) \bigg| \leq 2Lp^{\alpha}.
$$
\end{customlemma}

\begin{proof}
When $supp(\mu) = \mathcal{X}$, we have $\mu(B, r) > 0$ for each $r>0$. The result then follows from a loose analysis:
\begin{align*}
    \left| \eta \left( B'(x, r, z) \right)-\eta(B(x, r)) \right| &\leq  \left| \eta \left( B'(x, r, z) \right)- \eta(x) \right| + \left |\eta(x) - \eta(B(x, r')) \right| \\
    &\leq  Lp^{\alpha} + Lp^{\alpha} \\
    &= 2Lp^{\alpha}, \\
\end{align*}
where the the second inequality is implied by the definition of smoothness, as noted (in the case of non-augmented balls with $r>0$) by \citep{CD2014}. We repeat the idea here for completeness, noting that 
\begin{align*}
    \left | \eta \left( B'(x, r, z) \right) - \eta(x) \right |  &= \left |\bigg( \frac{1}{\mu' (B'(x, r, z))}\int_{B'(x, r, z)} \eta(x', z') d \mu'   \bigg) - \eta(x) \right|\\
    &= \left | \frac{1}{\mu' (B'(x, r, z))}\int_{B'(x, r, z)} \left(\eta(x', z')  - \eta(x)\right) d \mu'  \right| \\ 
    &= \left | \frac{1}{\mu' (B'(x, r, z))}\int_{B'(x, r, z)} \left(\eta(x')  - \eta(x) \right) d \mu' \right| \\ 
    &\leq \frac{1}{\mu'  (B'(x, r, z))}\int_{B'(x, r, z)}  \left |\eta(x')  - \eta(x) \right| d\mu'  \\ 
    &\leq \frac{1}{\mu' (B'(x, r, z))}\int_{B'(x, r, z)}  L \mu(B^o(x,r))^{\alpha} d \mu' \\
    &= L \mu(B^o(x,r))^{\alpha} \\
    &\leq L p^{\alpha}. 
\end{align*}
The final line comes from \citep{CD2014} Lemma 4. Note that such statements can be made for open balls as well with the same proof technique. 
\end{proof}

\begin{customlemma}{6}\label{smooth effective interior}
Fix $p \in (0,1)$. Then if $\mu$ has full support, and $(\mathcal{X}, \rho, \mu)$ is $(\alpha, L)$-smooth, it holds that 
$$
\mathcal{X}^+_{p, \Delta} =  \left \{x \in \mathcal{X} : \forall r \leq r(x;p), \ \eta(B(x, r)) - 1/2 \geq \Delta \right \},
$$
and 
$$
\mathcal{X}^-_{p, \Delta} =  \left \{x \in \mathcal{X} : \forall r \leq r(x;p), \ 1/2 - \eta(B(x, r)) \geq \Delta \right \}.
$$
\end{customlemma}

\begin{proof}
By definition, when $supp(\mu) = \mathcal{X}$, we have $\mu(B, r) > 0$ for each $r>0$, and we have 
$$
\mathcal{X}^+_{p, \Delta} =  \left \{x \in \mathcal{X} : \eta(x)> 1/2 \wedge  \forall r \leq r(x;p), \ \eta(B(x, r)) - 1/2 \geq \Delta \right \},
$$
Thus, it suffices to show that $\eta(B(x, r)) - 1/2 \geq \Delta$ implies $\eta(x)>0$ under smoothness of $\eta$. If $\mu(B(x, 0))>0$, then this is trivial. So assume WLOG that $\mu(B(x, 0)) = 0$. By the smoothness, we have for each $r$ with $0<r \leq r(x;p)$ that 
$$
\left | \eta(B(x, r)) - \eta(x) \right| \leq L \mu(B^o(x, r))^{\alpha}.
$$
We can now write for such $r$ that 
\begin{align*}
1/2 < \Delta/2  + 1/2 &< \Delta + 1/2 \\
&\leq \eta(B(x, r)) \\
&\leq \left| \eta(B(x, r)) - \eta(x) \right| + \eta(x) \\
&\leq  L \mu(B^o(x, r))^{\alpha} + \eta(x) \\
&\leq  L \mu(B(x, r))^{\alpha} + \eta(x). 
\end{align*}
Thus we may conclude, for $0 < r \leq r(x;p)$, that
$$
\eta(x) > \Delta/2 - L \mu(B^o(x, r))^{\alpha}  + 1/2. 
$$
Thus, if there is some such $r$ such that $\Delta/2 \geq L \mu(B(x, r))^{\alpha}$, we will show $\eta(x) >1/2$. Let $B_n := B(x, r(x;p)/n)$ for $n = 1, 2, \dots$.  Then $\cap_{n \geq 1} B_n = \{x\}$.\footnote{One can show by continuity from below of the measure that $r(x;p) < \infty$ for $p<1$.} By continuity from above
 of the measure, we have that 
 $$
 \lim_{n \to \infty} \mu(B_n) = \mu\left( \cap_{n \geq 1} B_n \right) = \mu(\{x \}) = 0.
 $$ 
 Thus, there is some $N$ for which $n \geq N$ yields $\Delta/2 \geq  L \mu(B(x, r(x;p)/n))^{\alpha}$. Then $r =  r(x; p)/N$ in the above inequality implies that $\eta(x) > 1/2$. 
\end{proof}

The next result, due to \citep{BDF2019}, describes the uniform convergence of $k$-NN votes to the true underlying conditional probability in augmented balls. The utility of this result for us is to ensure that for any new $x \in \mathcal{X}$ which is a candidate for rejection sampling, we can be sure that looking at it's $k$-NN votes, we get an accurate representation of the conditional probability nearby.

\begin{customlemma}{7}\label{conditional_UC}
Let $\mathcal{B}'$ be the set of all augmented balls, i.e. $\mathcal{B}' =  \left \{ B(x, r, z) : x \in \mathcal{X}, \ r >0 , \ z  \in[0,1]  \right\}$, and suppose this set has VC-dimension $d$ in $\mathcal{X} \times [0,1]$. Fix $k < n \in \mathbb{N}$ and $\delta \in (0,1)$. With probability $\geq 1-\delta^2/32$ over the draw of $n$ i.i.d. samples from $\mathbb{P}$, we have that
$$
\sup_{B' \in \mathcal{B}'} |\eta(B') - \hat{\eta}(B')| < c_0 \sqrt{\frac{d\log(n) + \log(1/\delta)}{k}}
$$
for a universal constant $c_0$.
\end{customlemma}

\begin{proof}
This is the content of Theorem 8 of \citep{BDF2019}.
\end{proof}

Finally, we show that the effective boundary after $n$ samples is contained in our ``estimated hard region'' with high probability when $\eta$ is smooth enough. We implicitly assume $\mu$ has full support in the following 
to simplify the result statements. 
\vspace{10mm}
\begin{customlemma}{1}\label{effective_upper_bound}
%
Fix $k <n \in \mathbb{N}$, and $\delta \in (0,1)$. Suppose we estimate 
$$
\Hard_{n, k} = \left \{ x  \mid \left| \hat{\eta}(B_{(k+1)}'(x)) - 1/2 \right| < 3 \Delta_{\textnormal{UC}}  \right\}.
$$
Then if $\eta$ is sufficiently $(\alpha, L)$-smooth in the sense that
$$
2L \bigg( \pnp \bigg)^{\alpha} \leq \Delta_{\textnormal{UC}}, 
$$
then with probability $\geq 1-\delta^2/16$ over the draw of $n$ $i.i.d.$ samples from $\mathbb{P}$, 
$$
\partial_{\pn, \Delta} \subseteq \Hard_{n, k}.
$$
\end{customlemma}

\begin{proof}
We will show that when uniform convergence takes effect for the empirical estimates of the mass of balls in the space $\mathcal{X}$, as well as for conditional probabilities in augmented balls, that $\Hard_{n,k}$ approximates the effective boundary from the outside. The necessary uniform convergence statements will hold with high probability over the sampling giving the result. 

Fix an arbitrary sample $\Sampn$. Suppose that for this sample, it holds for each $x \in \mathcal{X}$ that 
$$
 \left| \hat{\eta}(B'_{(k+1)}(x)) -\eta(B'_{(k+1)}(x)) \right| < \Delta_{\textnormal{UC}}.
 $$
Suppose further that for each $x \in \mathcal{X}$, it holds that
$$
B_{(k+1)}(x) \subseteq B_{\pnp}(x).
$$ 
In this case, by Lemma \ref{estimation_condition_measure_robust_under_smoothness} and the assumption on the smoothness of the distribution, we have for any $x\in \mathcal{X}$ and any $r \leq r(x; \pnp)$ that
 $$
 \bigg|\eta(B'_{(k+1)}(x)) - \eta(B(x, r))\bigg| \leq 2L \left(\pnp \right)^{\alpha} \leq  \Delta_{\textnormal{UC}}.
 $$
Suppose that $\hat{\eta}(B^{'}_{(k+1)}(x)) \geq 1/2$. When $\eta$ is sufficiently smooth, and the two notions of uniform convergence hold, it holds by the triangle inequality that for any $r \leq r(x; \pn)$, 
\begin{align*}
 \hat{\eta}(B'_{(k+1)}(x)) - 1/2  &\leq \left| \hat{\eta}(B'_{(k+1)}(x))  - \eta(B'_{(k+1)}(x)) \right| + \eta(B'_{(k+1)}(x)) - 1/2  \\
&\leq \left| \hat{\eta}(B'_{(k+1)}(x)) - \eta(B'_{(k+1)}(x)) \right| + \left| \eta(B'_{(k+1)}(x)) -  \eta(B(x, r))\right|  \\
& \ \ \ \ \ \ \ \  +   \eta(B(x, r)) -  1/2 \\
&\leq  2\Delta_{\textnormal{UC}}  + \eta(B(x, r)) -  1/2, 
\end{align*}
as such an $r$ must respect $r \leq r(x;\pnp)$ given that $r(x; \pn)\leq r(x;\pnp)$). Put differently, 
$$  
 \hat{\eta}(B'_{(k+1)}(x)) - 1/2   -  2\Delta_{\textnormal{UC}}   \leq  \eta(B(x, r)) -  1/2.
$$
Suppose $x \notin \Hard_{n, k}$.  Then by definition and our assumption on the estimate, $3\Delta_{\textnormal{UC}} \leq \left|\hat{\eta}(B'_{(k+1)}(x)) - \frac{1}{2} \right| = \hat{\eta}(B^{'}_{(k+1)}(x)) - 1/2$, which, given that $\Delta < \Delta_{\textnormal{UC}}$ (as $c_0$ is sufficiently large), implies that 
$$
\Delta \leq  \eta(B(x, r))- 1/2 .
$$
Because this statement holds for each $r \leq r(x; \pn)$, it holds by Lemma \ref{smooth effective interior} and the definition of the effective boundary that $x \notin \partial_{\pn, \Delta}$. The case where $\hat{\eta}(B^{'}_{(k+1)}(x)) < 1/2$ can be handled analogously.

By Lemmas \ref{conditional_UC} and \ref{knn_estimates_larger_rp_ball_from_inside} respectively, the first two events take place each with probability $\geq 1-\delta^2/32$ over the draw of the sample, and so by the union bound, the conclusion holds with probability $\geq 1-\delta^2/16$. 
\end{proof}

\subsection{Formulation of Rejection Region, Controlling its Measure}
As discussed in the main text, the actual region used for rejection sampling is a larger cousin of the estimated hard region. We will see that this allows us to more easily apply $k$-NN analysis techniques from the literature, as it will allow us to avoid analyzing cases where the $k$-NN form a biased estimate of the conditional probability of 1 inside any given nearest neighbors ball.  A vital part of guaranteeing speedup under our scheme is an upper bound on the measure of this acceptance region. In this section, we show how to form a distributional quantity that provides such a bound.

Recall that given an estimated hard region $\Hard_{n, k}$ and parameter $\zeta$, we say its associated augmented estimate is the set
$$
\Hard^+_{n, k, \zeta} := \bigcup_{x \in \Hard_{n, k}} B^o_{(\bar{k}_{\zeta} + 1)}(x), 
$$
and the ``easy region'' $\Easy_{n, k}$ is defined in the following manner:
\begin{align*}
\Easy_{n, k}^+ &:= \left \{x \mid  \forall r \leq r(x; \pnpp),  \ \eta(B(x, r)) - 1/2 \geq 5 \Delta_{\textnormal{UC}} \right\} \\ 
\Easy_{n, k}^- &:= \left \{x \mid  \forall r \leq r(x; \pnpp),  \ 1/2 - \eta(B(x, r))  \geq 5 \Delta_{\textnormal{UC}} \right\} \\
\Easy_{n, k} &:= \Easy_{n, k}^+ \cup \Easy_{n, k}^-.
\end{align*}
We first show that when the sample allows us to approximate probability radii well, and $\eta$ is sufficiently smooth, membership in the acceptance region means that there is some point in the estimated hard region with a similar local conditional probability.  
\vspace{7mm}

\begin{customlemma}{8}\label{augmented_still_contentious}
Fix $\delta \in (0, 1)$, $\zeta \geq 0$, $k, n \in \mathbb{N}$ such that $\bar{k}_{\zeta} < n \in \mathbb{N}$, and a first round sample $\Sampn$. Suppose it holds for some $x \in \mathcal{X}$ that 
$$
x \in \Hard_{n, k, \zeta}^{+},
$$
and the sample is such that for all $x' \in \mathcal{X}$, 
$$
B_{(k+1)}(x') \subseteq B_{\pnp}(x'), \  B_{(\bar{k}_\zeta + 1)}(x') \subseteq B_{\pnpp}(x').
$$
Then if $\eta$ is $(\alpha, L)$-smooth, there exists $\tilde{x} \in \Hard_{n, k}$ for which
$$
\left |\eta \left(B_{(\bar{k}_\zeta+1)}(x) \right)-\eta\left(B'_{(k+1)}(\tilde{x}) \right) \right| \leq 3L (\pnpp)^{\alpha}.
$$
\end{customlemma}

\begin{proof}
If $x \in \Hard_{n, k, \zeta}^{+}$ and $B_{(\bar{k}_\zeta + 1)}(x') \subseteq B_{\pnpp}(x')$, then $\exists \tilde{x} \in \Hard_{n, \Delta}$ for which $\rho(\tilde{x}, x) \leq r_{\pnpp}(\tilde{x})$; by definition of the augmented estimate, there is some $\tilde{x} \in \Hard_{n,k}$ for which $x \in B^o_{(\bar{k}_{\zeta}+1)}(\tilde{x})$, and so because $B^o_{(\bar{k}_{\zeta}+1)}(\tilde{x}) \subseteq B_{(\bar{k}_{\zeta}+1)}(\tilde{x}) \subseteq B_{\pnpp}(\tilde{x})$, the condition $\rho(\tilde{x}, x) \leq r_{\pnpp}(\tilde{x})$ follows by definition. Now we may write, using the definition of smoothness and the properties of smoothness discussed in Lemma \ref{estimation_condition_measure_robust_under_smoothness}, 
\begin{align*}
    \left |\eta \left(B_{(\bar{k}_{\zeta}+1)}(x) \right)-\eta\left(B'_{(k+1)}(\tilde{x}) \right) \right| &\leq \bigg| \eta\left( B_{(\bar{k}_{\zeta}+1)}(x) \right)-\eta(x) \bigg| + \bigg|\eta(x)-\eta(\tilde{x})\bigg| + \bigg|\eta(\tilde{x}) - \eta(B'_{(k+1)}(\tilde{x}))\bigg|\\
    &\leq L (\pnpp)^{\alpha} + L(\pnpp)^{\alpha} + L(\pnp)^{\alpha} \\
    &\leq 3L (\pnpp)^{\alpha}.
\end{align*}
where the bounds on the first and second term of the first line come from the assumption that $\forall x \in \mathcal{X}$,   \ $B_{(\bar{k}_{\zeta}+1)}(x) \subseteq B_{\pnpp}(x)$, and the bound on the third term comes from the assumption that $\forall x \in \mathcal{X}$,  $B_{(k+1)}(x) \subseteq B_{\pnp}(x)$. 
\end{proof}

This result allows us to associate the easy region with the acceptance region of Algorithm 1. Lemma \ref{augmented_outside_estimate} gives us a natural bound on the measure of the acceptance region automatically.
\vspace{5mm}

\begin{customlemma}{9}\label{augmented_outside_estimate}
Fix $\delta \in (0, 1)$, $\zeta \geq 0$ and $k, n \in \mathbb{N}$ such that $\bar{k}_{\zeta} <n \in \mathbb{N}$. Suppose $\eta$ is sufficiently $(\alpha, L)$-smooth in the sense that 
$$
3L (\pnpp)^\alpha \leq \Delta_{\textnormal{UC}}.
$$
Then with probability $\geq 1 - 3\delta^2/32$ over a draw of $n$ i.i.d samples from $\mathbb{P}$, 
$$
 \Hard^+_{n,k, \zeta} \subseteq \mathcal{X} \setminus \Easy_{n, k}.
$$
\end{customlemma}

\begin{proof}
The proof technique is the same as that of Lemma \ref{effective_upper_bound}. We suppose $x \in \Hard^+_{n,k, \zeta} $, and show that $x \notin \Easy_{n,k}$ when certain conditions hold; finally, we argue by a union bound that the conditions occur simultaneously with probability $\geq 1-3\delta^2/32$. 

Suppose it holds that, in addition to $3L (\pnpp)^\alpha \leq \Delta_{\textnormal{UC}}$, we have that for each $B' \in \mathcal{B}':= \{ B(x, r, z) : x\in \mathcal{X}, r >0, z\in [0,1
] \}$, that $|\eta(B') - \hat{\eta}(B') | < \Delta_{\textnormal{UC}}$, and that for each $x \in \mathcal{X}$,  both $B_{(\bar{k}_{\zeta}+1)}(x) \subseteq B_{\pnpp}(x)$ and $B_{(k+1)}(x) \subseteq B_{\pnp}(x)$. In this case, the following string of inequalities holds, where $\tilde{x} \in \Hard_{n, k}$. 
\begin{align*}
\bigg| \eta\left(B_{(\bar{k}_{\zeta}+1)}(x) \right) - \frac{1}{2} \bigg| &\leq \bigg| \eta \left (B'_{(k+1)}(\tilde{x}) \right) -\frac{1}{2} \bigg| + 3L (\pnpp)^\alpha \\ 
&\leq \left| \eta\left( B'_{(k+1)}(\tilde{x}) \right) -\frac{1}{2} \right| + \Delta_{\textnormal{UC}}\\
&\leq \left | \hat{\eta}\left(B'_{(k+1)}(\tilde{x}) \right) - \frac{1}{2} \right| + 2\Delta_{\textnormal{UC}}  \\
&< 3 \Delta_{\textnormal{UC}} + 2\Delta_{\textnormal{UC}} \\
&= 5\Delta_{\textnormal{UC}}, 
\end{align*}
where the first line is a consequence of Lemma \ref{augmented_still_contentious} and the triangle inequality, the second is by $3L (\pnpp)^\alpha \leq \Delta$, the third is is by the uniform convergence of conditional probabilities in the augmented balls,  and the second to last is from the definition of what it means for $\tilde{x} \in \Hard_{n,k}$. This means that under these conditions,  $\rho(x, X_{\bar{k}_{\zeta}}(x))$ is a radius that is sufficiently contentious to make $x \notin \Easy_{n, k}$, given that in this case $\rho(x, X_{\bar{k}_{\zeta}}(x)) \leq r(x; \pnpp)$. By Lemmas 1, 2 and 4, and a union bound, the conditions hold with probability $\geq 1 - 3\delta^2/32$, and so the result holds as claimed.
\end{proof}

There are a lot of moving parts in our formulation of the acceptance region after the first $n$ samples. We introduce a couple of definitions to help summarize the events we need to take place over the first $n$ samples for our rejection sampling/ classification scheme to have desirable performance properties.  We then show that these events take place with high probability. 
\vspace{5mm}

\setcounter{definition}{8}
\begin{definition}
Given $\delta \in (0,1)$, a parameter setting of $\zeta \geq 0$ in Algorithm 1 and $k, n \in \mathbb{N}$ such that $\bar{k}_{\zeta} < n$, we define the event that the first round sample $\Sampn$ of size $n$  is ``good for estimation''  to be  
\begin{align*}
\textnormal{GE}(\Sampn) &:= \bigg\{\partial_{\pn, \Delta} \subseteq \Hard_{n,k} \subset  \Hard_{n,k,\zeta}^{+} \subseteq \mathcal{X} \setminus \Easy_{n,k} \land \ \forall x \in \mathcal{X}, \ B_{\pn}(x) \subset B^o_{(\bar{k}_{\zeta}+1)}(x)  \bigg\} .
\end{align*}
Analogously we say that the sample $\Sampn$ is ``bad for estimation'' when this event doesn't take place, and write
\begin{align*}
\textnormal{BE}(\Sampn) &:= \textnormal{GE}(\Sampn)^c.
\end{align*}
\end{definition}

\begin{customlemma}{10}\label{rare_bad_estimation}
Fix $k < n \in \mathbb{N}$. Suppose $\mu$ is $\zeta$-regular, and $\eta$ is sufficiently $(\alpha, L)$-smooth in the sense that 
$$
 3L (\pnpp)^\alpha \leq \Delta_{\textnormal{UC}}. 
$$
Then it holds that
$$
\mathbb{E}_{\Sampn \sim \mathbb{P}^{\bigotimes n}} \big[ \mathbbm{1} \ \textnormal{BE}(\Sampn) \big] \leq \delta^2/8.
$$
\end{customlemma}

\begin{proof}
By definition, if it holds that
$$
\partial_{\pn, \Delta} \subseteq \Hard_{n,k}, 
$$
and
$$
\Hard_{n,k, \zeta}^{+} \subseteq \mathcal{X} \setminus \Easy_{n,k}, 
$$
and 
$$
B_{\pn}(x) \subseteq B^o_{(\bar{k}_{\zeta} +1)}(x), 
$$
then the sample is `good for estimation.'' These events occur simultaneously with probability $\geq 1-\delta^2/8$ by Lemmas \ref{effective_upper_bound}, \ref{augmented_outside_estimate} and \ref{rp_balls_estimated_from_outside} respectively, and the union bound. This is because there are only 4 fundamental uniform convergence results on which these results are built, so the fact each of these holds with probability $\geq 1-\delta^2/32$ is sufficient for the result.
\end{proof} 

\subsection{Further Notation for Analyzing Second-Round Samples}
Roughly speaking, we have so far gathered the results required to ensure that the first round of sampling is likely to go to plan in terms of our estimate of what a good rejection region will be. What we are left to show is that if the first round goes well, then it is likely that the second round samples lead to an accurate classifier. 

To do this, we essentially treat prediction using the rejection samples as a new prediction problem in a smaller space. In this section, we introduce some definitions that help us to that end, as well as some complex bad events that describe how the modified $k$-NN classifier introduce in the Preliminaries can predict incorrectly. As we did with the estimation events, we will look to control the measure of these events over the two rounds of the sampling to obtain a bound on the probability that the modified classifier disagrees with the Bayes optimal. 

Formally, the classifier whose performance we wish to bound is:
\begin{align*}
\hat{g}(x) := 
  \begin{cases}
  \begin{cases}
    \mathbbm{1}[\frac{1}{k}\sum_{i=1}^k Y_{(i)}(x; \Hard^+_{n,k, \zeta}) \geq \frac{1}{2}] & \text{if } x \in \Hard_{n,k} \\  
    \mathbbm{1}[\frac{1}{k}\sum_{i=1}^k Y_{(i)}(x) \geq \frac{1}{2}] & \text{otherwise}
  \end{cases}
    &\text{if $\mu(\Hard^+_{n,k, \zeta})>0$}\\
    \begin{cases}
       \mathbbm{1}[\frac{1}{k}\sum_{i=1}^k Y_{(i)}(x) \geq \frac{1}{2}] 
    \end{cases}
    &\text{if $\mu(\Hard^+_{n,k, \zeta})=0$}
  \end{cases}
\end{align*}
A related useful notion will be that of the effective boundary in a subspace. To introduce this, we must introduce some related ideas.

\begin{definition}
Given $x \in \mathcal{X}$, $r \geq 0$, and $S \subseteq \mathcal{X}$, the ``closed ball of radius $r$ inside the set $S$'' is the set
$$
B^{\cap S}(x, r) := \left \{ x' \in S : \rho(x, x') \leq r \right \} = B(x, r) \cap S.
$$
 \end{definition}

\begin{definition}
For a given $p \in(0,1)$ and measurable $S \subseteq \mathcal{X}$, the ``probability radius in $S$'' of some $x \in S$ is 
$$
r^{\cap S}(x; p) := \inf \left  \{r \mid \mu \left(B^{\cap S}(x,r) \right) \geq p \right \}, 
$$
whenever $\mu(S) \geq p$, and $\infty$ whenever $\mu(S)< p$.  
\end{definition}

\begin{definition}
For a given $p \in(0,1)$,  $S \subseteq \mathcal{X}$, the ``effective interior in $S$'' is the set 
\begin{align*}
\mathcal{X}^S_{p, \Delta} &:= \left \{x \in S : \eta(x) > 1/2  \ \wedge \ \forall r \leq r^{\cap S}(x; p),  \  \ \eta(B^{\cap S}(x, r)) - 1/2  \geq \Delta   \right \} \\ 
& \hspace{10mm} \cup  \left \{x \in S : \eta(x) < 1/2 \ \wedge \ \forall r \leq r^{\cap S}(x; p),  \  \ 1/2 - \eta(B^{\cap S}(x, r))   \geq \Delta   \right \}. \\
\end{align*}
\end{definition}

\begin{customlemma}{11}\label{rSp behaves nice}
Fix $p' \in (0, 1)$ and a measurable $S \subseteq \mathcal{X}$.  If it further holds holds that $B(x, r(x;p')) \subseteq S$, then
$$
 r^{\cap S}(x; p') = r(x; p').
$$
\end{customlemma}
\begin{proof}
Let $R := \{r : \mu(B(x, r) \geq p' \}$ and $R^{\cap S} := \{r : \mu(B^{\cap S}(x, r) \geq p' \}$. Note that if $r \geq r(x;p')$, then $r \in R$ and $r \in R^{\cap S}$, as $B^{\cap S}(x, r(x;p')) = B(x, r(x;p'))$. If $r < r(x; p')$, then it holds that $\mu(B(x, r)) = \mu(B^{\cap S}(x, r)) < p'$, or by definition of the infimum, $r(x; p')$ is not a lower bound for each element in $R$, and so cannot be the infimum. Thus $R= R^{\cap S}$, so they have the same infimum. 
 
\end{proof}

\begin{customlemma}{12}\label{smooth boundary small space}
Fix $p' \in (0, 1)$ and a measurable $S \subseteq \mathcal{X}$.  If it further holds holds that $B(x, r(x;p')) \subseteq S$, $\mu$ has full support, and $(\eta, \rho, \mu)$ is $(\alpha, L)$-smooth, then the condition
$$
\forall r \leq r^{\cap S}(x; p),  \  \ \eta(B^{\cap S}(x, r)) - 1/2  \geq \Delta
$$
implies that $\eta(x) > 1/2$.
\end{customlemma}
\begin{proof}
When $B(x, r(x;p')) \subseteq S$, we have that $r(x;p') = r^{\cap S}(x ;p')$ by Lemma \ref{rSp behaves nice}, and for each $r \leq  r^{\cap S}(x ;p')$, we have $B^{\cap S}(x, r) = B(x, r)$. Thus, we may repeat the argument of Lemma \ref{smooth effective interior}. 
\end{proof}

\begin{definition}
For a given $p \in(0,1)$,  $S \subseteq \mathcal{X}$, the ``effective boundary in $S$'' is the set 
$$
\partial^S_{p, \Delta} := S \setminus \mathcal{X}^S_{p, \Delta}.
$$
\end{definition}
\noindent

\subsubsection{Bad Prediction Events for the Modified $k$-NN Classifier}
The following introduce the ways in which the modified $k$-NN classifier can err. Essentially, they are more complex versions of the original ideas of \citep{CD2014}. The idea of their analysis is that for a new point in instance space that is not in the effective boundary, the $k$-NN classifier can only make a mistake if one of the two events take place:  the $k$-NN of $x$ are further away than one would expect with $n$ samples, or the empirical estimate of the conditional probability in the $k$-NN ball has a large deviation. The simplicity of this analysis follows directly from the definition of the effective boundary. For us, our modified $k$-NN classifier is essentially two $k$-NN classifiers predicting on different probability spaces. Thus, our idea is to repeat these analytical ideas on two separate spaces at once. The event of non-local neighbors or non-representative $k$-NN votes is a compound event that takes place when non-locality or non-representativeness takes place in either of the two spaces. 

Fix $x \in \mathcal{X}$, $p' \in (0,1)$, $\delta \in (0, 1)$, $\pi \in (0, 1)$,  $\zeta \geq 0$, $k, n, m \in \mathbb{N}$ such that $\bar{k}_{\zeta} <n$, and a sample $\mathcal{S} = \Sampn \cup \Sampm$.  Denote the event that the sample  is ``non-local'' relative to $x$ laying in the estimated hard region to be 
$$
\textnormal{NL}_H(x, \mathcal{S}) :=  \left\{\rho\left(x, X_{(k+1)}(x; \Hard^+_{n,k, \zeta})\right) > r^{\cap \Hard^+_{n,k,\zeta}}(x; \pnm), \ x \in \Hard_{n, k} \right \}, 
$$
where $X_{(k+1)}(x; \Hard^+_{n,k, \zeta})$ is the $k+1^{st}$ nearest instance to $x$ drawn as a rejection sample from $\Hard^+_{n,k, \zeta}$. 
Similarly, we consider the event that the sample $\mathcal{S}$ is ``non-local'' relative to $x$ laying outside the estimated hard region to be 
$$
\textnormal{NL}_{\neg H}(x, \mathcal{S}) :=  \left\{\rho\left(x, X_{(k+1)}(x)\right) > r(x; \pn), \ x \notin \Hard_{n, k} \right\}.
$$
We characterize the event that the nearest neighbors of $x$ are ``not local'' via the somewhat complex looking functions
\begin{align*}
\mathcal{I}_{\textnormal{NL}, H}(x, \mathcal{S}) :=
  \begin{cases}
    \mathbbm{1} \  \textnormal{NL}_H(x, \mathcal{S}) \cdot \mathbbm{1} \textnormal{GE}(\Sampn) \cdot \mathbbm{1} \{x \in \Hard_{n,k} \}  & \text{if $\mu(\Hard^+_{n,k,\zeta}) > 0$} \\
    0 & \text{if $\mu(\Hard^+_{n,k,\zeta}) = 0$}
  \end{cases}
\end{align*}
and
$$
\mathcal{I}_{\textnormal{NL}}(x , \mathcal{S}) := \mathcal{I}_{\textnormal{NL}, H}(x, \mathcal{S})  +   \mathbbm{1} \  \textnormal{NL}_{\neg H}(x, \mathcal{S}).
$$
The use of $\mathcal{I}_H$ has the utility of creating an indicator-like function that is formally well-defined in the case that $\mu(\Hard^+_{n,k,\zeta})=0$, in which case there are no training points drawn as rejection samples. 

We also define some more complex events corresponding to non-representative estimation of conditional probabilities. Let the event that the conditional probability of an augmented ball defined by the $k+1$-NN to $x \in \mathcal{X}$ is underestimated by a margin $\geq \Delta$ based on the samples from $\mathbb{P}$  be 
$$
\textnormal{BV}_{\neg H, 1}(x, \mathcal{S}) := \left \{  \eta \left(B'_{(k+1)}(x) \right) -  \hat{\eta} \left(B'_{(k+1)}(x) \right) \geq  \Delta \right \}. 
$$
Here, $\hat{\eta}$ refers to estimation of the conditional probability in this ball using samples from $\mathbb{P}$ only. Analogously, when the estimation is done with samples from the second sampling round, we have
$$
\textnormal{BV}_{H, 1}(x, \mathcal{S}) := \left \{ \eta \left(B^{\prime}_{(k+1)}(x; \Hard^+_{n,k, \zeta}) \right)- \hat{\eta} \left(B^{\prime}_{(k+1)}(x; \Hard^+_{n,k, \zeta}) \right) \geq  \Delta \right \}.
$$
Here, $B^{\prime}_{(k+1)}(x; \Hard^+_{n,k, \zeta})$ denotes the ``augmented ball in $\Hard^+_{n,k, \zeta}$'' of $k+1$-NN to $x$ drawn as a rejection sample from $\Hard^+_{n,k, \zeta}$, and $\hat{\eta}$ refers to the empirical estimation of the conditional probability in that ball using rejection samples. Formally speaking, 
\begin{align*}
B^{\prime}_{(k+1)}(x; \Hard^+_{n,k, \zeta})  :=&  \bigg\{ x' \in  \Hard^+_{n,k, \zeta}, z' \in [0, 1] : \rho(x, x') <  \rho\left(x, X_{(k+1)}(x; \Hard^+_{n,k, \zeta})\right) \veebar \\
 & \hspace{10mm} \rho(x, x') =  \rho\left(x, X_{(k+1)}(x; \Hard^+_{n,k, \zeta})\right) \land z'< Z_{(k+1)}(x; \Hard^+_{n,k, \zeta}) \bigg\}.
\end{align*}
Similarly, we consider the event that the conditional probability is locally overestimated by a margin $\geq \Delta$ by samples from $\mathbb{P}$, 
$$
\textnormal{BV}_{\neg H, 0}(x, \mathcal{S}) := \left \{ \hat{\eta}(B'_{(k+1)}(x)) - \eta(B'_{(k+1)}(x)) \geq \Delta \right \}, 
$$
and by samples from $\Hard_{n,k,\zeta}^+$, 
$$
\textnormal{BV}_{H, 0}(x, \mathcal{S}) := \left \{ \hat{\eta}(B'_{(k+1)}(x; \Hard^+_{n,k, \zeta})) - \eta(B'_{(k+1)}(x; \Hard^+_{n,k, \zeta})) \geq \Delta \right \}.
$$
As in the case of formalizing the non-locality events above, we define
\begin{align*}
\mathcal{I}_{\textnormal{BV}, H, 0}(x, \mathcal{S}) :=
  \begin{cases}
    \mathbbm{1} \textnormal{BV}_{H, 0}(x, \mathcal{S}) \cdot \mathbbm{1} \{x \in \Hard_{n,k} \} \cdot \mathbbm{1} \textnormal{GE}(\Sampn)   & \text{if $\mu(\Hard^+_{n,k,\zeta}) > 0$} \\
    0 & \text{if $\mu(\Hard^+_{n,k,\zeta}) = 0$}
  \end{cases}
\end{align*}
and 
\begin{align*}
\mathcal{I}_{\textnormal{BV}, H, 1}(x, \mathcal{S}) :=
  \begin{cases}
    \mathbbm{1} \textnormal{BV}_{H, 1}(x, \mathcal{S}) \cdot \mathbbm{1} \{x \in \Hard_{n,k} \} \cdot \mathbbm{1} \textnormal{GE}(\Sampn)  & \text{if $\mu(\Hard^+_{n,k,\zeta}) > 0$} \\
    0 & \text{if $\mu(\Hard^+_{n,k,\zeta}) = 0$}
  \end{cases}
\end{align*}
The ``indicator'' on the event that the voting ``goes badly'' around $x$ is then defined as 
\begin{align*}
\mathcal{I}_{\textnormal{BV}}(x, \mathcal{S}) := & \mathbbm{1} \{ x \in \mathcal{X}_0 \} \bigg(\mathcal{I}_{\textnormal{BV}, H, 0}(x, \mathcal{S})   + \mathbbm{1} \  \textnormal{BV}_{\neg H, 0}(x, \mathcal{S}) \bigg) + \\
& \mathbbm{1} \{ x \in \mathcal{X}_1 \} \bigg( \mathcal{I}_{\textnormal{BV}, H, 1}(x, \mathcal{S}) + \mathbbm{1} \ \textnormal{BV}_{\neg H, 1}(x, \mathcal{S}) \bigg), 
\end{align*}
where $\mathcal{X}_0 := \left \{x \in \mathcal{X} \mid \eta(x) < \frac{1}{2} \right \}$, and $\mathcal{X}_1 := \left  \{x \in \mathcal{X} \mid \eta(x) \geq \frac{1}{2}  \right \} $.

\subsection{Towards a Finite Sample Bound}
Given these event definitions, we will show that their measures can be bounded over the sampling, leading to a finite sample bound for our active strategy. First, we show how to bound the measure of a ``bad vote'' for our modified $k$-NN classifier. In subsequent lemma, we show that the non-locality event is also a rare event. We assume WLOG that $\pi m$ is an integer to facilitate the analysis.

\begin{customlemma}{13}\label{badvote_rare}
Fix $\delta \in (0,1)$, $k, n, m \in \mathbb{N}$, and $x \in \mathcal{X}$. Suppose a sample $\mathcal{S}$ is drawn according to Algorithm 1 run with parameters $\pi \in (0, 1)$ and $\zeta \geq 0$, with $k < \min(\pi m, (1-\pi)m)$ and $\bar{k}_{\zeta} < n$. Then 
$$
\mathbb{E}_{\mathcal{S}} \big[ \mathcal{I}_{\textnormal{BV}}(x, \mathcal{S}) \big] \leq 2\exp(-2k\Delta^2) . 
$$
\end{customlemma}

\begin{proof}
Suppose that $x \in \mathcal{X}_0$. Then it holds that
\begin{align*}
\mathcal{I}_{\textnormal{BV}}(x, \mathcal{S})  &= \mathcal{I}_{\textnormal{BV}, H, 0}(x, \mathcal{S})  + \mathbbm{1} \  \textnormal{BV}_{\neg H, 0}(x, \mathcal{S}).\\ 
\end{align*}
We have via the argument in \citep{CD2014} Lemma 10 (tightened via a one-sided Hoeffding bound) that 
$$
 \mathbb{E}_{\mathcal{S}} \big[\mathbbm{1} \  \textnormal{BV}_{\neg H,0}(x, \mathcal{S}) \big] \leq \exp(-2k\Delta^2).
$$
The idea here is that averaging over the labels of the $k$-NN of $x$ forms an unbiased estimate of $\eta \left(B'(x, X_{(k+1)}(x), Z_{(k+1)}(x)) \right)$, and so the probability that this estimate deviates by more than $\Delta$ from it's true mean can be controlled via a Hoeffding bound.

The first term can be handled similarly if we consider the sampling procedure in the two stages. Let the random variable
$$
G := \mathbbm{1}[x \in \Hard_{n,k}] \cdot \mathbbm{1} \ \textnormal{GE}(\Sampn)  \cdot \mathbbm{1} \{\mu(\Hard^+_{n,k,\zeta})>0 \},
$$
and suppose after the first $n$ samples, we have $G=1$. In this case, the classifier $\hat{g}$ will only use the second $m$ samples to make predictions on $x$,  It it treats prediction on $x$ as prediction in a new space defined by conditioning on $\Hard_{n,k, \zeta}^+$. 
As $k< \min(\pi m, (1-\pi)m)$, we are guaranteed that $X_{(k+1)}(x; \Hard^+_{n,k})$ exists after rejection sampling. Thus, we can reapply the argument given to bound the second term, which as in \citep{CD2014}, works in metric spaces with Borel regular measures: $(\Hard^+_{n,k,\zeta}, \rho)$ is a metric space, and $\mu$ conditioned on $\Hard^+_{n,k,\zeta}$ is easily checked to be a Borel regular measure. 

In somewhat more formal language, assume WLOG that $\textnormal{Pr}_{\Sampn \sim \mathbb{P}^{\otimes n}} \left( G = 0 \right) < 1$, as otherwise $\mathcal{I}_{\textnormal{BV}}(x, \mathcal{S}) = 0$ almost surely over the sampling, making the bound hold trivially. We then can somewhat informally apply the law of total expectation to compute the expectation of $\mathcal{I}_{\textnormal{BV}}(x, \mathcal{S})$. 
\begin{align*}
\mathbb{E}_{\Samp} \bigg[\mathcal{I}_{\textnormal{BV}, H, 0}(x, \mathcal{S}) \bigg] &= \textnormal{Pr}_{\Sampm}\left( \textnormal{BV}_{H, 0}(x, \mathcal{S}) | G = 1 \right) \cdot \textnormal{Pr}_{\Sampn \sim \mathbb{P}^{\otimes n}} \left( G = 1 \right) \\
&\leq  \textnormal{Pr}_{\Sampm}\left( \textnormal{BV}_{H, 0}x, \mathcal{S}) | G = 1 \right) \\
&\leq \exp(-2k\Delta^2), 
\end{align*}
where the last inequality follows from reapplying the argument made above.
 The full result follows from repeating the argument for the case that $x \in \mathcal{X}_1$, and noting that 
 $$
 2\exp \left(-2k\Delta^2 \right) \bigg( \mathbbm{1} \ \{x \in \mathcal{X}_0 \} + \mathbbm{1} \ \{x \in \mathcal{X}_1 \} \bigg) = 2\exp(-2k\Delta^2).
 $$
\end{proof}

\begin{customlemma}{14}\label{nonlocal_rare0}
Fix $x \in \mathcal{X}$,  $k, n, m \in \mathbb{N}$, and $p'', \gamma \in (0,1]$. Suppose $\mathcal{S}$ is drawn in accordance with Algorithm 1 run with parameters $\pi \in (0,1)$ and $\zeta \geq 0$,  such that $k \leq  (1-\gamma)  (n+\pi m) p'' $ and $\bar{k}_{\zeta} <n$. Then we have 
$$
\mathbb{E}_{\mathcal{S}} \bigg[ \mathbbm{1}_{\textnormal{NL}, \neg H}(x, \mathcal{S}) \bigg] \leq \exp(- (n+\pi m) p'' \gamma^2/2) \leq \exp(-k \gamma^2/2).
$$ 
\end{customlemma}

\begin{proof}
This is the content of \citep{CD2014} Lemma 9. During each run of the algorithm, $n + \pi m$ samples are drawn from $\mathbb{P}$. Thus, it suffices bound the probability of the event $\left\{\rho\left(x, X_{(k+1)}(x)\right) > r(x; p'') \right\}$ when
taking $n+\pi m $ i.i.d. samples from $\mathbb{P}$. As noted by \citep{CD2014} Lemma 22, it holds that $\mu(B(x, p'')) \geq p''$. Thus, by the multiplicative Chernoff bound, the probability that $\leq k \leq (1-\gamma) (n+\pi m) p''$ samples fall in the the ball $B_{p''}(x)$ is $\leq  \exp(-\gamma^2 (n+\pi m) p''/2)$.
\end{proof}

\begin{customlemma}{15}\label{nonlocal_rare}
Fix $x \in \mathcal{X}$,  $k, n, m \in \mathbb{N}$, $p' \in(0, c_{\delta}k/n]$, and $\gamma \in (0,1]$. Suppose $\mathcal{S}$ is drawn in accordance with Algorithm 1 run with parameters $\pi \in (0,1)$ and $\zeta \geq 0$,  such that $k \leq  (1-\gamma)  (1-\pi m) \frac{p'}{\mu(\mathcal{X} \setminus \Easy_{n,k})}$, $k < \min(\pi m, (1-\pi)m)$, and $\bar{k}_{\zeta} <n$. Then we have 
$$
\mathbb{E}_{\mathcal{S}} \bigg[ \mathcal{I}_{\textnormal{NL}, H}(x, \mathcal{S}) \bigg] \leq \exp(-k \gamma^2/2).
$$ 
\end{customlemma}

\begin{proof} 
Again, consider the random variable $G := \mathbbm{1} \{x \in \Hard_{n,k} \} \cdot \mathbbm{1} \  \textnormal{GE}(\Sampn) \cdot \mathbbm{1} \{\mu(\Hard^+_{n,k,\zeta} >0 \}$. Assume WLOG that $\textnormal{Pr}_{\Sampn \sim \mathbb{P}^{\otimes n}} \left( G = 1 \right) > 0$, as otherwise we have $ \mathcal{I}_{\textnormal{NL}, H}(x, \mathcal{S}) = 0$ almost surely over the sampling from Algorithm 1, and the bound holds. Then, using the law of total expectation for discrete random variables, we have
\begin{align*}
\mathbb{E}_{\Samp} \bigg[ \mathcal{I}_{\textnormal{NL}, H}(x, \mathcal{S}) \bigg] &= \textnormal{Pr}_{\Sampm}\left( \textnormal{NL}_H(x, \mathcal{S}) | G = 1 \right) \cdot \textnormal{Pr}_{\Sampn \sim \mathbb{P}^{\otimes n}} \left( G = 1 \right) \\
&\leq  \textnormal{Pr}_{\Sampm}\left( \textnormal{NL}_H(x, \mathcal{S}) | G = 1 \right). 
\end{align*}
When $G=1$, it holds that $B_{\pnm}(x) \subseteq B_{\pn}(x) \subseteq B^o_{(\bar{k}_{\zeta}+1)}(x)$, and so, the probability of any one of the $(1-\pi)m$ accepted samples from the second round falling in the ball 
$$B^{\cap \Hard^+_{n,k, \zeta}} \left(x, r^{\cap \Hard^+_{n,k, \zeta}}(x; \pnm) \right)$$
is $\geq \frac{\pnm}{\mu(\mathcal{X} \setminus \Easy_{n,k})}$. Thus, this conditional probability is the probability $\leq k$ successes occur out of $(1-\pi)m$ independent Bernoulli trials, each with parameter $\geq \frac{\pnm}{\mu(\mathcal{X} \setminus \Easy_{n,k})}$. This is controlled as claimed in the lemma statement by the multiplicative Chernoff bound.
\end{proof}

We need one final result before giving the main, finite sample bound. This is a technical result that will allow us to relate the effective boundary inside the acceptance region to the effective boundary in the instance space $\mathcal{X}$, which gives a simplification of our finite sample bound that facilitates comparison with passive guarantees.

\begin{customlemma}{15}\label{simplified_boundary}
Fix $x \in \mathcal{X}$, $\zeta \geq 0$, $k, n, m \in \mathbb{N}$, such that $\bar{k}_{\zeta} < n$ and  $\pnm \leq \pn$. For all first round samples $ \Sampn$, it holds that 
$$
\left \{ x \in \partial^{\Hard^+_{n,k, \zeta}}_{\pnm, \Delta}, \ x \in  \Hard_{n,k} \right \} \cap \textnormal{GE}(\Sampn) \subseteq \bigg \{ x \in \partial_{\pnm, \Delta} \bigg \}.
$$
\end{customlemma}

\begin{proof}
If the sample is such that $\mathbbm{1}  \textnormal{GE} (\Sampn) = 1$ and it further holds that $x \in \Hard_{n,k}$, then it holds from the definition of the event that the sample is ``good for estimation'' that  
$$
B\left (x, r(x;\pn) \right) \subseteq \Hard^{+}_{n, k,\zeta}.
$$
By $\pnm \leq \pn$, we have  $r(x; \pnm) \leq r(x; \pn)$, and so it further holds that 
$$
B(x, r(x;\pnm)) \subseteq \Hard^{+}_{n, k, \zeta},
$$ 
and that $r^{\Hard_{n,k,\zeta}^+}(x; \pnm) = r(x; \pnm)$ by Lemma \ref{rSp behaves nice}. Suppose by contradiction that $x \notin \partial_{p', \Delta}$. Then WLOG, $x \in \mathcal{X}^+_{p', \Delta} = \{x : \forall r \leq r(x;p'), \ \eta(B(x, r)) -1/2 \geq \Delta \}$. 
Because $B(x, r(x;\pnm)) \subseteq \Hard^{+}_{n, k, \zeta}$, we must have for each $r \leq r(x;\pnm)$, 
$$
B(x, r)  = B^{\cap \Hard^{+}_{n, k, \zeta}}(x, r), 
$$
and so, this would imply via Lemma \ref{smooth boundary small space} that $x \not \in  \partial^{\Hard^+_{n,k, \zeta}}_{\pnm, \Delta}$, which is a contradiction.
\end{proof}

The main finite sample bound can now be stated. The analysis techniques come straight from \citep{CD2014}, where we have more complex ``bad events'' to ensure don't take place. 

\begin{customthm}{2}\label{finite_sample_bound}
Fix $\delta \in (0, 1)$, $\zeta \geq 0$,  $\pi \in (0,1)$, and $k, n, m \in \mathbb{N}$,  such that $\bar{k}_{\zeta} < n$ and $m (1-\pi) c_{\delta} \geq n c_{\delta/\sqrt{2}}$. Assume that $\mu$ is $\zeta$-regular, and $\eta$ is $(\alpha, L)$-smooth such that 
$$
3L (\pnpp)^\alpha \leq \Delta_{\textnormal{UC}}.
$$
Then, with probability $\geq 1-\delta$ over the draw of at most $n+m$ samples from Algorithm 1 run with parameters $\zeta$ and $\pi$, we have 
$$
\textnormal{Pr}_{X \sim \mu} \bigg( g^{*}(X) \ne \hat{g}(X) \bigg) \leq \mu(\partial_{\pnm, \Delta}) + \delta, 
$$
where
$$
\pnm = c_{\delta/\sqrt{2}} \cdot \frac{k}{ \frac{1-\pi}{\mu(\mathcal{X} \setminus \Easy_{n,k})}m}.
$$
\end{customthm}

\begin{proof}
Fix $x_0 \in \mathcal{X}$ arbitrarily.  We assert that for every valid sample\footnote{A sample $\mathcal{S}$ is valid if has $(1-\pi)m$ points designated as rejection sampled points and $n + m \pi$ points designated as samples from $\mathbb{P}$ when $\Sampn$ is such that $\mu(\Hard^+_{n,k,\zeta})>0$, and simply has $n+m \pi$ points designated as being from $\mathbb{P}$ otherwise.} $\mathcal{S} = \Sampn \cup \Sampm$, the following bound holds:
\begin{align*}
    \mathbbm{1}  \left \{g^{*}(x_0) \ne \hat{g}(x_0) \right \} 
    \leq&  \mathbbm{1} \left \{ x_0 \in \partial^{\Hard^+_{n, \Delta}}_{\pnm, \Delta}, \ x_0 \in \Hard_{n, k} \right\} \cdot \mathbbm{1} \textnormal{GE}(\Sampn) +  \\ 
    & \mathbbm{1} \left \{ x_0 \in \partial_{c_{\delta}k/(n+\pi m), \Delta}, \  x_0 \notin \Hard_{n, \Delta} \right\} \cdot \mathbbm{1} \textnormal{GE}(\Sampn) + \\ 
    & \mathbbm{1} \ \textnormal{BE}(\Sampn) +  \\
    & \mathcal{I}_{\textnormal{NL}}(x_0, \mathcal{S})  +  \\
    & \mathcal{I}_{\textnormal{BV}}(x_0, \mathcal{S}) .  
\end{align*}
To see this, fix a valid sample $\mathcal{S}$, and WLOG assume $x_0 \in \mathcal{X}_0$.  Note that it's impossible that $\mathcal{S}$ is such that $\mathbbm{1} \ \textnormal{GE}(\Sampn) = 1$, $x_0 \in \Hard_{n,k}$, and $\mu(\Hard^+_{n,k,\zeta})=0$, as  $x_0 \in \Hard_{n,k}$ and $\mathbbm{1} \ \textnormal{GE}(\Sampn) = 1$, implies $\mu(\Hard^+_{n,k,\zeta}) > c_{\delta} k/n$.

Assume $\mathcal{S}$ is such that none of the first three terms in the upper bound evaluate to 1, or the bound holds automatically for this sample on $x_0$. Suppose WLOG that we have $x_0 \in \Hard_{n, k}$. Thus, by the above, we have  $\mu(\Hard^+_{n,k, \zeta}) >0$, and so we have $(1-\pi)m$ rejection samples from $\Hard^+_{n,k,\zeta}$. Assume $\mathcal{I}_{\textnormal{NL}}(x_0, \mathcal{S}) =0$, or the bound is validated for this $\mathcal{S}$, and note that because  $\mu(\Hard^+_{n,k, \zeta})>0$, we have  $\mathbbm{1}\textnormal{NL}_H(x, \mathcal{S})$ is both well-defined and equals $0$ in this case.

 In this case, because we have that 
$$
x_0 \notin \partial^{\Hard^+_{n, \Delta}}_{\pnm, \Delta},
$$
by definition, we have
$$
 \forall r \leq r^{\cap \Hard^+_{n,k, \zeta}}(x_0; \pnm), \ \frac{1}{2}  - \Delta \geq  \eta\left(B^{\cap \Hard^+_{n,k, \zeta}}(x_0, r) \right),
$$
and thus under smoothness, we have $1/2 -\Delta  \geq \eta(x_0)$. Thus, the only way our classifier disagrees with the Bayes optimal classifier is if the voting overestimates the conditional probability of the augmented $k+1$-NN ball by a margin greater than $\Delta$, i.e. $\mathbbm{1} \ \textnormal{BV}_{H, 0}(x_0, \mathcal{S}) =1$, validating the bound for this sample $\mathcal{S}$; see the proof of \citep{CD2014} Theorem 1 for some formalism regarding the applicability of augmented balls in estimating conditional probabilities in closed balls in non-augmented space.\footnote{The idea is the conditional probability of an augmented ball is a convex combination of open and closed balls, and by definition of event that the sample is not 'non'-local and the effective boundary, the conditional probability in the open and closed versions nearest neighbor ball both are bounded away from $1/2$ by a margin of $\Delta$.} 

The full claim can be proved with identical arguments, by working through variations of the WLOG claims. Note that by the fact that  $m (1-\pi) c_{\delta} \geq n c_{\delta/\sqrt{2}}$ and $k < \bar{k}_{\zeta}< n$, the augmented hard region can always be formed, and there are always $k+1$-nearest neighbors taken as rejection samples, if rejection samples are indeed taken. 

The above bound can be further simplified to the following:
\begin{align*}
    \mathbbm{1}  \left \{g^{*}(x_0) \ne \hat{g}(x_0) \right \} 
    \leq&  \mathbbm{1} \left \{ x_0 \in \partial_{\pnm, \Delta} \right\} +  \\ 
    & \mathbbm{1} \ \textnormal{BE}(\Sampn) +  \\
    & \mathcal{I}_{\textnormal{NL}}(x_0, \mathcal{S})  +  \\
    & \mathcal{I}_{\textnormal{BV}}(x_0, \mathcal{S}) .  
\end{align*}
To see this why this holds, two observations need to be made. Firstly, the first term of the original bound admits transformation into the first term of this simplified bound by Lemma \ref{simplified_boundary} when the condition $m (1-\pi) c_{\delta} \geq n c_{\delta/\sqrt{2}}$ holds (as this condition implies that $p' \leq c_{\delta}k/n$). Further, the second term of the original bound can be eliminated entirely. This elimination comes from the fact that
$$
\left \{ x_0 \in \partial_{c_{\delta}k/(n+\pi m), \Delta}, \ x_0 \notin \Hard_{n,k} \right\}  = \emptyset,
$$ 
as when $\mathbbm{1} \textnormal{GE}(\Sampn) = 1$ and $\eta$ is sufficiently smooth in the sense of the assumption, we have 
$$
\partial_{c_{\delta}k/(n+\pi m), \Delta} \subseteq \partial_{\pn, \Delta} \subseteq \Hard_{n, k}.
$$
Now, by Lemma \ref{rare_bad_estimation} and the $\zeta$-regularity, we have that 
$$
\mathbb{E}_{\Sampn}\left [  \mathbbm{1} \ \textnormal{BE}(\Sampn) \right] \leq \delta^2/8.
$$
By Lemma \ref{nonlocal_rare}, and by the setting of the parameter $\pnm$ and $\pn$, we have 
$$
\mathbb{E}_{\mathcal{S}}  \left[ \mathcal{I}_{\textnormal{NL}}(x_0, \mathcal{S})  \right] \leq 3\delta^2/8,
$$
where by $\mathbb{E}_{\Sampn, \Sampm} $ we mean the expectation over the sampling scheme defined by Algorithm 1. By Lemma \ref{badvote_rare} and the choice of $\Delta$, 
$$
\mathbb{E}_{\mathcal{S}}  \left[  \mathcal{I}_{\textnormal{BV}}(x_0, \mathcal{S}) \right] \leq \delta^2/2.
$$
Thus, for a new $X_0 \sim \mu$, 
$$
\mathbb{E}_{X_0} \mathbb{E}_{\mathcal{S}} \big[  \mathbbm{1} \ \textnormal{BE}(\mathcal{S}) +   \mathcal{I}_{\textnormal{NL}}(x_0, \mathcal{S}) +   \mathcal{I}_{\textnormal{BV}}(x_0, \mathcal{S}) \big] \leq \delta^2
$$
Thus, by Markov's inequality, their sum is $\leq\delta$ with probability $\geq 1-\delta$. This gives the result given that $\mathbb{E}_{X_0 \sim \mu} \big[ \mathbbm{1} \left \{ X_0 \in \partial_{\pnm, \Delta} \right\} \big] = \mu( \partial_{\pnm, \Delta} )$. 
\end{proof}

\vspace{5mm}

\begin{customcor}{1}
Suppose the assumptions of Theorem 2 hold are satisfied by $\alpha$, $L$, $\delta$, $\pi$, $\zeta$, $k$ for schedules of first and second round samples $n$, and $m = \Theta(n)$, i.e. for all pairs $(m(i), n(i))$ for $i$ in some unbounded $D \subset \mathbb{N}$. If in addition $\mathbb{P}$ satisfies the $\beta$-margin condition, with probability $\geq 1-\delta$ over the draw of at most $n+m$ training samples from Algorithm 1,
\begin{align*}
&\textnormal{Pr}_{X \sim \mu} \bigg( g^{*}(X) \ne \hat{g}(X) \bigg)  \leq  O \left(\left(  \sqrt{\frac{\log(1/\delta)}{k}} + \left(c_{\delta/\sqrt{2}} \cdot \frac{k}{n} \cdot \mu(\mathcal{X} \setminus \Easy_{n,k})\right)^\alpha  \right)^\beta \right) + \delta, 
\end{align*}
where 
$$
\mu(\mathcal{X} \setminus \Easy_{n,k}) \leq O \left( \left( \sqrt{\frac{d \log(n) + \log(1/\delta)}{k}} + \left(c_{\delta} \cdot \frac{k}{n} + \sqrt{\frac{d + \log(1/\delta)}{n}} \right)^{\alpha} \right)^{\beta} \right).
$$
The corresponding guarantee for the \citep{CD2014} analysis of passively trained nearest neighbors states that when $\beta$-margin condition holds, with probability $\geq 1-\delta$ over the draw of $n+m$ samples from $\mathbb{P}$,  
\begin{align*}
&\textnormal{Pr}_{X \sim \mu} \bigg( g^{*}(X) \ne g_{n+ m, k}(X) \bigg)  \leq  O \left(  \left( \sqrt{\frac{\log(1/\delta)}{k}} + \left(c_{\delta} \cdot \frac{k}{n}\right)^{\alpha} \right)^{\beta} \right)+ \delta.
\end{align*}
\end{customcor}

\begin{proof}
Under $(\alpha, L)$-smoothness and $\mu$ having full support, as noted by \citep{CD2014} Lemma 4, 
$$
\partial_{p, \Delta} \subset \left \{  x \in \mathcal{X} : \left| \eta(x) - \frac{1}{2} \right| \leq \Delta + Lp^{\alpha}  \right \}.
$$
The result, up to the decay of the mass of $\mathcal{X} \setminus \Easy_{n,k}$, follows from the direct application of the $\beta$-margin condition with margin $ \Delta + Lp^{\alpha} $, where $p$ is adjusted for the active case as in Theorem 2. 

The decay of the of the mass  of $\mathcal{X} \setminus \Easy_{n,k}$ follows the same lines. Note that from the definitions, 
$$
\partial_{\pnpp, 5\Delta_{\textnormal{UC}}} = \mathcal{X} \setminus \Easy_{n,k}, 
$$
and so the idea from above can be directly reapplied. 
\end{proof}

\subsection{General Consistency}
\begin{definition}
Given $k, n, m \in \mathbb{N}$,  $\zeta \geq 0$, $\pi \in (0, 1)$ such that $k< \min( \pi m, (1-\pi) m)$, a first round sample $\Sampn$, and a second round sample $\Sampm$, let  
\begin{align*}
\bar{g}_{m, k}(x) := 
  \begin{cases}
  \begin{cases}
    \mathbbm{1}[\frac{1}{k}\sum_{i=1}^k Y_{(i)}(x; \Hard^+_{n,k,\zeta}) \geq \frac{1}{2}] & \text{if } x \in \Hard_{n,k} \\  
    \mathbbm{1}[\frac{1}{k}\sum_{i=1}^k Y_{(i)}(x; 2\mathbb{P}) \geq \frac{1}{2}] & \text{otherwise }
  \end{cases}
    &\text{if $\mu( \Hard{H}^+_{n,k,\zeta})>0$}\\
    \begin{cases}
       \mathbbm{1}[\frac{1}{k}\sum_{i=1}^k Y_{(i)}(x; 2\mathbb{P}) \geq \frac{1}{2}] 
    \end{cases}
    &\text{if $\mu( \Hard^+_{n,k,\zeta})=0$}
  \end{cases}
\end{align*}
Here, $Y_{(i)}(x; 2\mathbb{P})$ denotes the label of the $i^{th}$ closest training point to $x$ drawn in the second round, not as a rejection samples from $\Hard^+_{n,k,\zeta}$, but from $\mathbb{P}$. 
\end{definition}

In what follows, given a confidence parameter $\delta$, a first round sample size $n$, and a schedule $k_r$, we use the notation $\Delta_{m} := \min(1/2, \sqrt{\log(2/\delta)/ k_{n+m}})$,  $p'_m := c_{\delta} k_{n+m}/(1-\pi) m$, $p_m := c_{\delta} k_{n+m}/\pi m$, and $\vec{p}_m := ( p'_m, p_m ) $. We also use the notation $\delta_0 := \{ x \in \mathcal{X} : \eta(x)=1/2 \}$ to denote the decision boundary.

\begin{customlemma}{16}\label{second_round_dominates}
Fix $n \in \mathbb{N}$, $\delta \in (0, 1)$, $ \pi \in (0,1), \zeta \geq 0$ and a schedule $k_{r}$ with the properties  $k_{r}/r \to 0$ and $k_r \to \infty$ as $r \to \infty$ such that $\bar{k}_{\zeta}< n$ when $k=k_n$.  Then there is some $M$ such that $m \geq M$ implies both that $\bar{g}_m$ is well-defined, and that for all first round samples $\Sampn$, with probability $\geq 1-\delta$ over the draw of $\Sampm$ according to the second sampling round of Algorithm 1 run with parameters $\pi$ and $\zeta$, 
$$
\textnormal{Pr}_{X \sim \mu} \bigg( \bar{g}_m(X) \ne \hat{g}(X), \ \eta(X) \ne 1/2 \bigg) \leq \mu(\partial_{\vec{p}_m, 2\Deltam} \setminus \partial_0) + \delta.
$$ 
\end{customlemma}

\begin{proof}
Given the schedule $k_r/r \to 0$, there is some $M_0$ for which $k_{n+m} < \min( \pi m, (1-\pi) m)$ for $m \geq M_0$. Thus, for $m \geq M_0$, we have that there are at least $k_{n+m}$ samples drawn in the second round from $\mathbb{P}$ (and as rejection samples if $\Sampn$ so dictates), and so $\bar{g}_m$ is well-defined. 
Further, because $k_r \to \infty$ as $r \to \infty$, there is some $M_1$ for which $m \geq M_1$ implies that $n/k_{n+m} < \sqrt{\log(2/\delta)/ k_{n+m}}$, which for some $M_2$ is the value of $\Deltam$ for $m \geq M_2$, again given that $k_r \to \infty$. Let $M= \max(M_0, M_1, M_2)$ so that for $m \geq M$, all of these properties hold.

Fix $x_0 \in \mathcal{X}$ and $\Sampn$ arbitrarily, and assume $m\geq M$.  Denote via $B'$ the augmented ball 
$$
B'\left(x, \rho(x_0, X_{(k_{n+m} +1)}(x_0; 2 \mathbb{P})), Z_{(k_{n+m} +1)}(x_0; 2 \mathbb{P}) \right).
$$ 
We claim the following bound holds for each possible valid second round sample $\Sampm$:
\begin{align*}
\mathbbm{1} \left \{ \bar{g}_m(x_0) \ne \hat{g}(x_0), \eta(x_0) \ne 1/2 \right \} \leq &\mathbbm{1} \left\{ x_0 \in \partial_{\vec{p}_m, 2 \Deltam} \setminus \partial_0  \right \} +  \\
&\mathbbm{1} \left \{\rho \left(x_0, X_{(k_{n+m} +1)}(x_0; 2 \mathbb{P}) \right) > r(x_0;  p_m) \right \} + \\
&\mathbbm{1} \left \{  \left| \hat{\eta}(B') - \eta(B') \right| \geq \Deltam    \right \}.
\end{align*}
Assume first that $\Sampn$ is such that $\mu(\Hard_{n,k, \zeta}) > 0$. If further it holds that $x_0 \in \Hard_{n, k}$ as defined by $\Sampn$, then by definition, $\bar{g}_m(x_0) = \hat{g}(x_0)$, and so the bound holds automatically for this sample pair $\Sampn$, $\Sampm$. Thus, assume for now that  $\mu(\Hard_{n,k, \zeta}) > 0$ and $x_0 \notin \Hard_{n, k}$. 

Further assume that $x_0 \notin \partial_{\vec{p}_m, 2 \Deltam} \setminus \partial_0$, and in particular that $x_0 \notin \partial_0$ and $x_0 \notin  \partial_{\vec{p}_m, 2 \Deltam}$, or the bound holds. We can also assume
 $\rho(x_0, X_{(k_{n+m} +1)}(x_0; 2 \mathbb{P})) \leq r(x_0;  p_m)$, or again, the bound holds. Then, because $x_0 \notin \partial_{\vec{p}_m, 2 \Deltam}$, we have $\eta(B') \geq 1/2 + 2\Deltam$ and $\eta(x) > \frac{1}{2}$. Assume also that  $\left| \hat{\eta}(B') - \eta(B') \right| < \Deltam$, or again the bound holds. In this case, we have 
$\hat{\eta}(B') \geq 1/2 + \Deltam$, and so $\bar{g}_m(x_0) = g^*(x_0) = 1$ by Lemmas 24 and 25 of \citep{CD2014}. 

Let $\tilde{B}'$ denote $B'(x, \rho(x_0, X_{(k_{n+m} +1)}(x_0)), Z_{(k_{n+m} +1)}(x_0)))$, the augmented ball formed of $k_{n+m}+1$ nearest neighbors in the first round sample $\Sampn$ union the $\pi m$ samples drawn from $\mathbb{P}$ in the second round; the nearest $k_{n+m}$ of these are the neighbors used by $\hat{g}$ for prediction on $x_0$. We claim that for $m \geq M$, when none of the indicators in the above upper bound on the disagreement indicator of $\bar{g}_m$ and $\hat{g}$ are equal to 1, the number of neighbors of $x_0$ from the first round sample $\Sampn$ is not large enough for $\hat{g}(x_0) \ne g^*(x_0)$. To see this, note that for an arbitrary binary vector $v$ of length $n$, 
\begin{align*}
\left | \hat{\eta}(B')  - \hat{\eta}(\tilde{B}') \right |&=  \left| \frac{1}{k_{n+m}} \sum_{i = 1}^{k_{n+m}} Y_{(i)}(x_0; 2\mathbb{P}) -  Y_{(i)}(x_0)\right| \\
&\leq   \frac{1}{k_{n+m}} \sum_{i = k_{n+m} - n +1}^{k_{n+m}}  \left| Y_{(i)}(x_0; 2\mathbb{P}) -  v_i \right| \\
&\leq \frac{n}{k_{n+m}} \\
&< \Deltam.
\end{align*}
That is to say, for any given $x_0$, the difference in the nearest neighbor votes between the two training sets is maximized when the votes of the $n$ furthest neighbors from $x$ that were drawn in the second round from $\mathbb{P}$ are replaced by votes from nearest neighbors in $\Sampn$. Thus, $\hat{g}(x_0) =  g^*(x_0)$, validating the bound for this sample pair. 

If $\Sampn$ is such that  $\mu(\Hard^+_{n, k, \zeta}) = 0$, then $\bar{g}_{m,k}$ uses only the second round samples from $\mathbb{P}$, whereas $\hat{g}$ uses both round samples from $\mathbb{P}$; in other words, predication takes place for both classifiers exactly as it would for the case where $\mu(\Hard^+_{n, k, \zeta}) > 0$ and $x_0 \notin \Hard_{n,k}$. Thus, analysis above then directly applies, validating the bound for such samples. 

The result follows by taking expectation with respect to $\pi m$ samples from $\mathbb{P}$ using the techniques of Lemmas \ref{badvote_rare} and \ref{nonlocal_rare0}  and applying Markov as in the proof of Theorem \ref{finite_sample_bound}. 
\end{proof}

\begin{definition}
For $p \in (0,1)$, the ``effective interior'' of \citep{CD2014} is the set 
$$
\mathcal{X}_{p, \Delta} := \mathcal{X} \setminus \partial_{p, \Delta}.
$$
\end{definition}

\begin{definition}
For a given set $S \subseteq \mathcal{X}$, and for $p \in (0,1)$, let the ``probability $p$-interior'' of $S$ be 
$$
\textnormal{I}_p(S) := \left \{ x \mid B\left(x, r(x;p) \right) \subseteq S \right \}.
$$
\end{definition}

\begin{definition}
For $\vec{p} = (p' , p) \in (0,1) \times (0,1)$, and given regions $S, S^+ \subseteq \mathcal{X}$, let the ``effective boundary in $\mathcal{X}$ relative to $S$ and $S^+$ and $\vec{p}$'' be 
$$
\tilde{\partial}^{S, S^+}_{\vec{p}, \Delta} : = \mathcal{X} \setminus \bigg( \big(\mathcal{X}_{p', \Delta} \cap \textnormal{I}_{p'}(S^+) \cap S \big) \sqcup \big( \mathcal{X}_{p, \Delta} \cap S^c \big)\bigg).
$$ 
\end{definition}

\begin{customlemma}{17}\label{second_round_right}
Pick $\delta \in (0, 1)$, $\zeta \geq 0$, $\pi \in (0,1)$, and $k, n, m \in \mathbb{N}$ such that $\bar{k}_{\zeta}< n$ and $k < \min(\pi m, (1-\pi) m)$. Fix $\Sampn$, and suppose $\Sampm$ is drawn according to the second sampling round of Algorithm 1 with parameters $\pi$ and $\zeta$. Then with probability $\geq 1-\delta$ over the draw of $\Sampm$,
$$
\textnormal{Pr}_{X \sim \mu} \bigg( g^*(X) \ne \bar{g}_{m,k}(X), \ \eta(X) \ne 1/2 \bigg) \leq \mu \left(\tilde{\partial}_{\vec{p}_m, \Delta}^{\Hard_{n,k}, \Hard^+_{n,k, \zeta}} \setminus \partial_0 \right) + \delta.
$$
\end{customlemma}

\begin{proof}
The proof is very similar to that of Theorem 2. Fix $x_0 \in \mathcal{X}$ and a valid sample $\mathcal{S}$, and assume first that $\Sampn$ is such that $\mu(\Hard^+_{n,k,\zeta})>0$. Define the event 
$$
\textnormal{NL}_{H, \textnormal{I}}(x_0, \mathcal{S}) := \left \{ \rho(x_0, X_{(k+1)}(x_0; \Hard^+_{n_,k,\zeta})) > r(x_0; p'_m), \ x_0 \in \textnormal{I}_{p'_m}(\Hard^+_{n,k, \zeta}), \ x_0 \in \Hard_{n,k}\right\},
$$
the event 
$$
\textnormal{NL}_{\neg H 2} (x_0, \mathcal{S}) := \left \{ \rho(x_0, X_{(k+1)}(x_0; 2\mathbb{P})) > r(x_0; p_m) ,\  x_0 \in \Hard_{n,k}^c \right\} ,
$$
and summarize via 
$$
\textnormal{NL}_{\textnormal{I}}(x_0, \mathcal{S}) :=\textnormal{NL}_{H, I}(x_0, \mathcal{S})  \cup \textnormal{NL}_{\neg H 2}(x_0, \mathcal{S}).
$$
Note that these events are well-defined; $\Hard^+_{n, k, \zeta}$ can be formed because $k_{\zeta} < n$, and there are always at least $k+1$ points drawn as rejection samples from $\Hard^+_{n, k, \zeta}$ and as second round samples from $\mathbb {P}$ given that $k < \min(\pi m, (1-\pi) m)$ and $\mu(\Hard_{n, k, \zeta}^+) > 0$. Further define
\begin{align*}
\textnormal{BV}_{\neg H 2, 0}(x_0, \Sampm) &:=    \bigg\{ \hat{\eta}(B'(x, \rho(x_0, X_{(k +1)}(x_0; 2 \mathbb{P}), Z_{(k+1)}(x_0; 2 \mathbb{P}))  \\
& \hspace{15mm}- \eta(B'(x, \rho(x_0, X_{(k+1)}(x_0; 2 \mathbb{P}), Z_{(k +1)}(x_0; 2 \mathbb{P}))) \geq \Delta \bigg \},   \\
\textnormal{BV}_{\neg H 2, 1}(x_0, \Sampm) &:=    \bigg\{ \eta(B'(x, \rho(x_0, X_{(k +1)}(x_0; 2 \mathbb{P}), Z_{(k+1)}(x_0; 2 \mathbb{P}))  \\
& \hspace{15mm} - \hat{\eta}(B'(x, \rho(x_0, X_{(k+1)}(x_0; 2 \mathbb{P}), Z_{(k +1)}(x_0; 2 \mathbb{P}))) \geq \Delta \bigg \}.  \\
\end{align*}
Again, these events are well-defined, this time by $k < \min(\pi m, (1-\pi) m)$. Finally, define
\begin{align*}
\mathbbm{1} \textnormal{BV}_{\textnormal{I}}(x_0, \Sampm) := & \mathbbm{1} \{ x_0 \in \mathcal{X}_0 \} \bigg( \mathbbm{1} \textnormal{BV}_{H, 0}(x_0, \Sampm)  \cdot \mathbbm{1} \{x_0 \in \textnormal{I}_{p'_m}(\Hard^+_{n,k, \zeta}) \} \cdot \mathbbm{1}\{x_0 \in \Hard_{n,k} \} + \mathbbm{1} \  \textnormal{BV}_{\neg H 2, 0}(x_0, \Sampm) \bigg) + \\
& \mathbbm{1} \{ x_0 \in \mathcal{X}_1 \} \bigg( \mathbbm{1} \textnormal{BV}_{H, 1}(x_0, \Sampm) \cdot \mathbbm{1} \{x_0 \in \textnormal{I}_{p'_m}(\Hard^+_{n,k, \zeta}) \} \cdot \mathbbm{1}\{x_0 \in \Hard_{n,k} \} + \mathbbm{1} \ \textnormal{BV}_{\neg H 2, 1}(x_0, \Sampm) \bigg).
\end{align*}
Similarly to Theorem \ref{finite_sample_bound}, we show that for any $x_0 \in \mathcal{X}$, and any valid second round sample $\Sampm$, the complete sample $\mathcal{S} = \Sampn \cup \Sampm$ of size $n+m$ respects 
\begin{align*}
    \mathbbm{1}  \left \{g^{*}(x_0) \ne \bar{g}_{m,k}(x_0), \ x_0 \notin \partial_0 \right \} 
    \leq&  \mathbbm{1} \left \{ x_0 \in \tilde{\partial}_{\vec{p}_m, \Delta}^{\Hard_{n,k}, \Hard^+_{n,k, \zeta}} \setminus \partial_0 \right\} +  \\
    &\mathbbm{1}  \textnormal{NL}_{\textnormal{I}}(x_0, \mathcal{S}) +   \\
    & \mathbbm{1} \textnormal{BV}_{\textnormal{I}}(x_0, \mathcal{S}).
\end{align*}
The reasoning is analogous.  Suppose that
$$
x_0 \notin  \tilde{\partial}_{\vec{p}_m, \Delta}^{\Hard_{n,k}, \Hard^+_{n,k, \zeta}} \setminus \partial_0;
$$ 
otherwise, the bound holds. In this case, we can assume $x_0 \notin \partial_0$, or the left hand side evaluates to 0 and the bound holds. Then WLOG, we have
$x_0 \in \textnormal{I}_{p'_m}(\Hard^+_{n,k,\zeta}) \cap \mathcal{X}_{p'_m, \Delta} \cap \Hard_{n,k}$. Then if $k$ or more of the second round rejection samples fall in $B(x_0, r(x_0; p'_m))$, then the only way $\bar{g}_m$ classifies differently from the Bayes optimal is for the voting of the $k$ neighbors drawn from the second round to be non-representative by a margin $\geq \Delta$. 

As before, we bound the probability of these events under a draw from the second round of Algorithm 1 given $\Sampn$. First, consider the event $\textnormal{NL}_{H, \textnormal{I}}(x_0, \mathcal{S})$.  Because $x_0 \in \textnormal{I}_{p'_m}(\Hard^+_{n,k, \zeta})$, $B(x_0, r(x_0; p'_m))$ is contained in the acceptance region. Then we can borrow the idea of Lemma \ref{nonlocal_rare} to say that with probability $\geq 1- \delta^2/4$,  we have that $\geq k $ of the points drawn as rejection samples fall in $B(x_0, r(x_0; p'_m))$. The probability that the even $\textnormal{NL}_{\neg H 2, \textnormal{I}}(x_0, \mathcal{S})$ can be handled exactly the same way.

Finally, whether prediction is taking place in $\Hard_{n,k}$ or its complement, the probability that the voting is non-representative of the conditional probability of 1 in an augmented ball by a margin $\geq \Delta$ is bounded by $\delta^2/2$ by the argument of Lemma \ref{badvote_rare}. One can finish the result the same way as in Theorem \ref{finite_sample_bound} by using Markov's inequality.

In the case that $\Sampn$ is such that $\mu(\Hard^+_{n,k, \zeta}) = 0$, consider the event
\begin{align*}
\mathbbm{1} \textnormal{BV}_{\neg H, \textnormal{I}}(x_0, \Sampm) := & \mathbbm{1} \{ x_0 \in \mathcal{X}_0 \} \cdot \mathbbm{1} \  \textnormal{BV}_{\neg H 2, 0}(x_0, \Sampm) + \\
& \mathbbm{1} \{ x_0 \in \mathcal{X}_1 \} \cdot \mathbbm{1} \ \textnormal{BV}_{\neg H 2, 1}(x_0, \Sampm).
\end{align*}
We note that the following bound holds, using the same reasoning as above, given that $\textnormal{I}_{p'_m}(\Hard^+_{n, k, \zeta}) =  \emptyset$, consider the event:
\begin{align*}
    \mathbbm{1}  \left \{g^{*}(x_0) \ne \bar{g}_{m,k}(x_0), \ x_0 \notin \partial_0 \right \} 
    \leq&  \mathbbm{1} \left \{ x_0 \in \tilde{\partial}_{\vec{p}, \Delta}^{\Hard_{n,k}, \Hard^+_{n,k,\zeta}} \setminus \partial_0 \right\} +  \\
    &\mathbbm{1}  \textnormal{NL}_{\neg H 2}(x_0, \mathcal{S}) +   \\
    & \mathbbm{1} \textnormal{BV}_{\neg H, \textnormal{I}}(x_0, \Sampm).
\end{align*}
The expectation is similarly controlled, leading to the general result. 
\end{proof}

\begin{customlemma}{18}\label{risk_upper}
For any classifiers $g, h : \mathcal{X} \to \mathcal{Y}$, it holds that 
$$
\mathbb{P}\left( g(X) \ne Y \right) - \mathbb{P}\left( h(X) \ne Y \right) \leq \textnormal{Pr}_{X \sim \mu} \left( g(X) \ne h(X), \ \eta(X) \ne 1/2 \right).
$$
\end{customlemma}

\begin{proof}
Let $R(f)(x) := \mathbb{P}\left( f(x) \ne Y \mid X=x \right)$. Then the risk $\mathbb{P}\left( f(X) \ne Y \right) = \mathbb{E}_{X \sim \mu} [R(f)(X)]$. Note that for all classifiers $f$, $R(f)(x) = 1/2$ if $x$ such that $\eta(x)=1/2$. If $\mu (\{\eta(x)\ne 1/2 \}) = 0$, then  $\mathbb{E}_{X \sim \mu} [R(h)(X)] = \mathbb{E}_{X \sim \mu} [R(g)(X)]  =1/2$, and so the bound holds. If $\mu (\{\eta(x)\ne 1/2 \}) =1$, then 
\begin{align*}
\mathbb{P}\left( g(X) \ne Y \right) - \mathbb{P}\left( h(X) \ne Y \right) &= \mathbb{E}_{X \sim \mu} [(R(g)(X) - R(h)(X)) \cdot \mathbbm{1} \{\eta(X) \ne 1/2 \}] \\
&\leq \mathbb{E}_{X \sim \mu} [\left| R(g)(X) - R(h)(X) \right| \cdot \mathbbm{1} \{\eta(X) \ne 1/2 \}] \\
&\leq \textnormal{Pr}_{X \sim \mu} \left( g(X) \ne h(X), \ \eta(X) \ne 1/2 \right),\\
\end{align*}
as integration is defined up to null sets, and so the bound holds. Otherwise, one can decompose the expectation and use standard rules for conditional expectation with respect to a $\sigma$-field generated by a partition of the sample space.
\end{proof}

\begin{customlemma}{19}\label{boundary_shrinks}
Suppose $(\mathcal{X}, \mu, \rho)$ satisfies the Lebesgue Differentiation Theorem. Fix $\Sampn$ such that for each $x \in \mathcal{X}$, we have $\rho(x, X_{\bar{k}_{\zeta}+1}(x))>0$. Then there exists $\mathcal{X}_{\textnormal{null}} \subset \mathcal{X}$ with $\mu(\mathcal{X}_{\textnormal{null}})=0$, $\vec{p} \in (0,1) \times (0,1)$, and $\gamma \in (0, 1/2]$ such that 
$$
x \notin \mathcal{X}_{\textnormal{null}}, \  \eta(x) \ne \frac{1}{2} \implies  x \notin  \tilde{\partial}_{\vec{p}, \Delta}^{\Hard_{n,k}, \Hard^+_{n,k,\zeta}}.
$$
\end{customlemma}

\begin{proof}
Very similarly to \citep{CD2014} Lemma 12, set 
$$
\mathcal{X}_{\textnormal{null}} := \left\{ x \in \mathcal{X} : \lim_{r \downarrow 0} \eta \left (B(x, r) \right) \ne \eta(x) \right \}.
$$
By the Lebesgue Differentiation Theorem, $\mu(\mathcal{X}_{\textnormal{null}})= 0$. Note that the condition $r_0 := \rho(x, X_{\bar{k}_{\zeta}+1}(x))>0$ implies that $\Hard_{n,k} \subseteq \Hard^+_{n, k,\zeta}$, and more specifically, that $x\in \Hard_{n,k}$ has $B^o(x, r_0) \subseteq \Hard^+_{n,k,\zeta}$. 

Now fix $x \notin \mathcal{X}_{\textnormal{null}}$ with $\eta(x) \ne \frac{1}{2}$. Suppose WLOG that $\eta(x)>1/2$ and $x \in \Hard_{n,k}$. Let $\gamma = (\eta(x) - 1/2)/2 > 0$. By the fact that $\lim_{r \downarrow 0} \eta \left (B(x, r) \right) = \eta(x)$, it holds that there is some $r'_0$ for which $r \leq r'_0$ implies that $\eta(B(x, r)) - \frac{1}{2} \geq \gamma$. Further, by definition of $\Hard^+_{n,k,\zeta}$ and $r_0 >0$, it holds that $B(x, r_0/2) \subseteq \Hard^+_{n,k, \zeta}$, given that $B(x, r_0/2) \subseteq B^o(x, r_0)$. Thus, if we choose 
$$
p'' = \min\bigg( \mu(B(x, r'_0)) , \mu(B(x, r_0/2)) \bigg) > 0, 
$$ 
it holds that $x \in \Hard_{n,k} \cap \textnormal{I}_{p''}(\Hard_{n,k}) \cap \mathcal{X}_{p'', \gamma}$. 
\end{proof}

\begin{customlemma}{19}\label{boundary_shrinks2}
Suppose $(\mathcal{X}, \mu, \rho)$ satisfies the Lebesgue Differentiation Theorem, and fix $\delta \in (0, 1)$. For each $\Sampn$ such that for each $x \in \mathcal{X}$ we have $\rho(x, X_{\bar{k}_{\zeta}+1}(x))>0$, we have
$$
\lim_{m \to \infty} \mu \bigg(  \partial_{\vec{p}_m, 2\Delta_m} \setminus \partial_0 \bigg) = 0,
$$
and 
$$
\lim_{m \to \infty} \mu \bigg(  \tilde{\partial}^{\Hard_{n,k}, \Hard^+_{n,k,\zeta}}_{\vec{p}_m, \Delta_m} \setminus \partial_0 \bigg) = 0.
$$
\end{customlemma}

\begin{proof}
The first statement follows \citep{CD2014} Lemma 13; the proof of the second is nearly identical. By definition, $\vec{p}_m \downarrow (0,0)$ (coordinate-wise) and $\Delta_m \downarrow 0$. Let 
$$
A_m :=   \tilde{\partial}^{\Hard_{n,k}, \Hard^+_{n,k,\zeta}}_{\vec{p}_m, \Delta_m}  \setminus \partial_0.
$$
Then $A_1 \supset A_2 \supset \dots$. By Lemma \ref{boundary_shrinks},  for $x \in \mathcal{X} \setminus ( \mathcal{X}_{\textnormal{null}} \cup \partial_0)$, there is some $M$ such that for $m\geq M$, $x \notin   \tilde{\partial}^{\Hard_{n,k}, \Hard^+_{n,k,\zeta}}_{\vec{p}_m, \Delta_m} $. Thus 
$$
\cap_{m \geq 1}  A_m \subset \mathcal{X}_{\textnormal{null}}, 
$$
and so the result follows by the fact that $\mu(\mathcal{X}_{\textnormal{null}})=0$ and using continuity from above of the measure $\mu$.
\end{proof}

\begin{theorem}
Suppose $(\mathcal{X}, \rho, \mu)$ satisfies the Lebesgue Differentiation Theorem. Fix $n \in \mathbb{N}$, $\delta \in (0,1)$, $\pi \in (0,1)$ and $\zeta \geq 0$ for use as parameters in Algorithm 1. Suppose the schedule for $k$ respects the following conditions as $m \to \infty$:
\begin{align*}
k_{n+m} &\to  \infty \\
k_{n+m}/ (n+m) & \to 0, 
\end{align*}
and $\bar{k}_{\zeta} < n$ when $k = k_n$. Then for any first round sample $\Sampn$ such that for all $x \in \mathcal{X}$ we have $\rho(x, X_{(\bar{k}_\zeta + 1)}(x))>0$,  it holds that 
$$
\lim_{m \to \infty} \textnormal{Pr}_{\Sampm} \bigg( R(\hat{g}) - R^{*} > \epsilon \bigg) = 0,  
$$
that is, we converge to the Bayes-risk in probability over the at most $m$ samples drawn according to the second round of Algorithm 1.
\end{theorem}

\begin{proof}
Firstly, by Lemma \ref{risk_upper}, we have that both 
$$
R(\bar{g}_{m,k}) - R^* \leq \textnormal{Pr}_{X \sim \mu} \left( \eta(X) \ne \frac{1}{2}, \ \bar{g}_{m,k}(X) \ne g^*(X) \right), 
$$
$$
R(\hat{g}) - R(\bar{g}_{m,k}) \leq \textnormal{Pr}_{X \sim \mu} \left( \eta(X) \ne \frac{1}{2}, \ \bar{g}_{m,k}(X) \ne \hat{g}(X) \right).
$$
 Thus, by Lemmas \ref{second_round_right} and \ref{second_round_dominates}, for any $\Sampn$, we have for any choice of $\tilde{\delta}$, that there is some $M_0$ such that if $m \geq M_0$, it holds with probability $\geq 1-2\tilde{\delta}$ over the second sampling round that both
$$
R(\bar{g}_{m,k}) - R^* \leq \mu\left(\tilde{\partial}^{\Hard_{n,k}, \Hard^+_{n,k,\zeta}}_{\vec{p}_m, \Delta_m} \setminus \partial_0 \right) + \tilde{\delta}, 
$$
$$
R(\hat{g}) - R(\bar{g}_{m,k}) \leq  \mu\left(\partial_{p_m, 2\Delta_m} \setminus \partial_0 \right) + \tilde{\delta}
$$
Choose a schedule $\tilde{\delta}_m = \exp(-k_{n+m}^{1/2})$, and fix $\epsilon \in (0, 1)$ arbitrarily. Then because for each $x$ we have $x \ne X_{\bar{k}_{\zeta}+1}(x)$, we have by Lemma \ref{boundary_shrinks2} that there is some $M_1$ such that for $m \geq M_1$, both
$$
\mu \left(\tilde{\partial}^{\Hard_{n,k}, \Hard^+_{n,k,\zeta}}_{\vec{p}_m, \Delta_m} \setminus \partial_0 \right) \leq \epsilon/2, 
$$
$$
 \mu \left(\partial_{p_m, 2\Delta_m} \setminus \partial_0 \right)  \leq \epsilon/2 
$$
This implies that if $m \geq M:= \max(M_0, M_1)$, we have 
$$
\mathbb{P} \left( R(\hat{g})  - R^* > \epsilon  \right) \leq 2\tilde{\delta}_m.
$$
We then take the limit as $m \to \infty$.
\end{proof}


\end{document}